%% file: aistats_conference.tex
\documentclass[twoside]{article}

\usepackage[accepted]{aistats2026}

\usepackage[round]{natbib}
\renewcommand{\bibname}{References}
\renewcommand{\bibsection}{\subsubsection*{\bibname}}
\input{math_commands.tex}

\usepackage{hyperref}
\usepackage{url}
\usepackage{amsmath, amssymb, amsthm, bm}
\usepackage{graphicx}
\usepackage{url}
\usepackage{booktabs}
\usepackage{multirow}

\usepackage{thmtools}
\usepackage{cleveref}

\usepackage{enumitem}

\usepackage{xcolor}
\usepackage{subcaption}
\usepackage{url}

\newcommand{\reals}{\mathbb{R}}

\newcommand{\normadj}{\hat{\bm{A}}}
\newcommand{\matH}{\bm{H}}
\newcommand{\matW}{\bm{W}}
\newcommand{\matU}{\bm{U}}

\newcommand{\matV}{\bm{V}}
\newcommand{\matLambda}{\bm{\Lambda}}
\newcommand{\matVp}{\bm{V}_{P}}
\newcommand{\vectheta}{\bm{\theta}}
\newcommand{\gft}[1]{\widehat{#1}}

\declaretheorem[name=Definition]{definition}

\begin{document}

\runningauthor{Martirosyan, Malitesta, Talbot, Giraldo, Malliaros}

\twocolumn[

\aistatstitle{Generalization Bounds for Spectral GNNs via Fourier Domain Analysis}

\aistatsauthor{Vahan A. Martirosyan$^1$ \And Daniele Malitesta$^1$ \And Hugues Talbot$^1$} 
\vspace{0.1cm}
\aistatsauthor{Jhony H. Giraldo$^2$  \And Fragkiskos D. Malliaros$^1$}
\vspace{0.1cm}
\aistatsaddress{ $^1$Université Paris-Saclay, CentraleSupélec, Inria, France \\ $^2$LTCI, Télécom Paris, Institut Polytechnique de Paris, France }

]

\begin{abstract}

Spectral graph neural networks learn graph filters, but their behavior with increasing depth and polynomial order is not well understood. 
We analyze these models in the graph Fourier domain, where each layer becomes an element-wise frequency update, separating the fixed spectrum from trainable parameters and making depth and order explicit. 
In this setting, we show that Gaussian complexity is invariant under the Graph Fourier Transform, which allows us to derive data-dependent, depth, and order-aware generalization bounds together with stability estimates.
In the linear case, our bounds are tighter, and on real graphs, the data-dependent term correlates with the generalization gap across polynomial bases, highlighting practical choices that avoid frequency amplification across layers.

\end{abstract}

\section{Introduction}

Graph Neural Networks (GNNs) have become the leading paradigm for graph machine learning, achieving state-of-the-art results on tasks ranging from node classification to link prediction \citep{survey1, survey2}. 
A widely used class of these models, spectral GNNs, defines graph convolutions by applying filter operations in the graph's frequency domain \citep{shuman2012gft, sandryhaila2013discrete}.
Architectures such as ChebNet \citep{defferrard2016ChebyNet} and its recent advancements \citep{hariri2025returnchebnet}, GCNs \citep{kipf2017gcn}, and more recent models using Bernstein \citep{he2021bernnet} or Jacobi polynomials \citep{wang2022jacobi} can all be unified under a common framework of learnable polynomial graph filters (\Cref{sec:unified_framework}).

Despite good empirical results, the theoretical principles governing the generalization of GNNs remain incomplete \citep{Garg2020Generalization, Verma2019Stability}.
Clarifying the conditions under which these models generalize is important for designing robust architectures, particularly in the transductive setting where the training and test data are not i.i.d. \citep{ Esser2021LearningTheory, Vasileiou2025Survey}.
The core  challenge comes from the complex, non-i.i.d. nature of graph data.
Unlike traditional machine learning settings, the nodes are interconnected, and in the transductive setting, the labeled training set and unlabeled test set are inherently dependent.
This structure makes it difficult to directly apply standard learning-theoretic tools and complicates efforts to derive interpretable generalization bounds that can guide architectural design.
While significant research has focused on the expressive power of GNNs, often relating them to the Weisfeiler-Leman test \citep{morris2023wl}, understanding why these models generalize well from a small set of labeled nodes to the rest of the graph remains an active area of research \citep{Vasileiou2025Survey}.

In this paper, we analyze the generalization of multi-layer spectral GNNs in the transductive setting. We work in the graph Fourier domain, and by Lemma~\ref{lem:gci} the full transductive Gaussian complexity is unchanged. This basis change turns graph convolution into element-wise multiplication by the frequency response, which lets us separate graph spectrum from learnable parameters and then plug a Fourier-side complexity bound into our generalization gap theorem (Theorem~\ref{thm:gc_gap}).
Our main contributions are:
\begin{itemize}[leftmargin=0.5cm]
    \setlength\itemsep{-2pt}
  \item We formalize spectral GNNs with arbitrary polynomial bases (\Cref{sec:unified_framework}) and introduce the generalized Vandermonde matrix $\matVp$ (\Cref{def:gen_vand_matrix}) to represent frequency responses compactly (Eq. (\ref{eq:fourier_layer})), which simplifies the analysis.
  \item We derive a Fourier-domain, data-dependent Gaussian complexity bound for deep spectral GNNs that isolates the roles of the spectrum (via $\matVp$), basis/order, depth, and parameter norms (\Cref{sec:data_dependent_bound}, \Cref{thm:data_dependent_bound}). Substituting it into \Cref{thm:gc_gap} yields an explicit bound on the transductive generalization gap.
  \item For linear spectral GNNs we prove a sharper bound that makes the depth effect explicit through the basis amplification profile (\Cref{thm:tighter_bound_deep}).
  \item We bound the network Jacobian norm (\Cref{thm:jacobian_bound}) and show the same factors that control the gap also control worst-case sensitivity.
  \item The first-layer term in our bound leads to a simple energy-weighted frequency regularizer (Eq. (\ref{eq:ew_ratio_single})) that reduces the measured generalization gap and improves accuracy with stable bases (see \Cref{tab:gap} and \Cref{tab:accuracy}).
\end{itemize}
Our results provide a principle for architectural design: generalization is better when selecting polynomial bases that are spectrally stable and by ensuring that the learnable filters do not excessively amplify the dominant frequencies of the input signal.

\section{Related Work}
Theoretical analysis of GNN generalization has largely followed three directions:  analyzing algorithmic stability, using the PAC-Bayesian framework, and adapting classical tools like the Vapnik–Chervonenkis (VC) dimension and Rademacher complexity.

\paragraph{Stability, PAC-Bayesian, and Covering Number Bounds.}
Algorithmic stability offers an alternative view on generalization, connecting it to how much the model's output changes when the training set is perturbed. \cite{Verma2019Stability} first studied this for GCNs, deriving stability-based bounds that depend on the spectral norm of the graph convolution filter. More recently, concurrent work \citep{liu2025generalization} also uses stability to study spectral GNNs, deriving bounds in expectation over generative models (contextual stochastic block models) for single-layer monomial filters. While both of these works are limited to single-layer models, our approach evaluates capacity directly on any fixed graph instance, explicitly accommodating deep, multi-layer architectures and arbitrary polynomial bases. Though we rely on Gaussian complexity rather than algorithmic stability, we arrive at a similar core conclusion: spectral properties are key to generalization.

The PAC-Bayesian framework has also been successfully applied. \cite{Liao2021PACBayesian} established generalization bounds for GCNs and Message Passing Neural Networks (MPNNs) that depend on the spectral norms of the weight matrices and the maximum node degree. This line of work also emphasizes the role of parameter norms, a factor that appears explicitly in our bounds.

Another line of work studies generalization by placing a metric on the space of graphs and showing that MPNNs are Lipschitz (or uniformly continuous) with respect to this metric, which yields covering-number bounds \citep{vasileiou2025covered}. Unlike these graph-space results, we analyze a single fixed graph in the transductive setting.

\paragraph{VC Dimension and Rademacher Complexity.}
Initial theoretical studies focused on bounding the VC dimension of GNNs. \cite{Scarselli2018VCDimension} provided early bounds for a specific class of GNNs by using Pfaffian function theory. More recently, \cite{morris2023wl} connected the VC dimension directly to the expressive power of GNNs, showing it is related to the number of non-isomorphic graphs distinguishable by the Weisfeiler-Leman test. However, as noted by \cite{Esser2021LearningTheory}, VC dimension-based bounds for GNNs can often be loose or even trivial, limiting their practical utility for explaining generalization on a fixed graph.

To obtain more informative, data-dependent bounds, some studies have turned to Rademacher complexity. \cite{Garg2020Generalization} derived Rademacher complexity bounds for graph-level prediction tasks, highlighting dependencies on model depth and feature dimension.
For the transductive (node-level) setting central to our work, \cite{el2009transductive} developed Transductive Rademacher Complexity (TRC), which can provide meaningful bounds where VC dimension fails. 
Using TRC \citep{Esser2021LearningTheory} demonstrates the importance of the interaction term between the graph diffusion operator and the node features, a finding that is consistent with our data-dependent results. In our work we study the same transductive node setting, but in the graph Fourier domain, which allows for a more fine-grained decomposition of this interaction.

\section{Preliminaries}
\label{sec:preliminaries}

This section provides the foundational concepts for our analysis. We first define the notation and the formal problem setup for transductive node regression. We then explain the theoretical basis for our approach to generalization, detailing the relationship between the generalization gap and Gaussian complexity. Finally, we review the principles of graph Fourier analysis.

\paragraph{Notation and Problem Setup.}
We consider an undirected graph $G = (V, E)$, where $V$ is the set of $n = |V|$ nodes and $E$ is the set of edges. Each node is associated with an initial feature vector. These are collected in a feature matrix $\matH^{(0)} \in \reals^{n \times d_0}$, where $d_0$ is the dimensionality of the input features for each node.

For our analysis, we use the normalized graph Adjacency matrix $\normadj = \bm{D}^{-1/2} \bm{A} \bm{D}^{-1/2}$, where $\bm{A}$ is the adjacency matrix and $\bm{D}$ is the degree matrix. As a real symmetric matrix, $\normadj$ has an eigendecomposition $\normadj = \matU \matLambda \matU^\top$. The matrix $\matU = [\bm{u}_1, \dots, \bm{u}_n]$ contains the orthonormal eigenvectors, and $\matLambda = \text{diag}(\lambda_1, \dots, \lambda_n)$ contains the corresponding real eigenvalues, $-1 \le \lambda_1 \le \dots \le \lambda_n \le 1$.
The eigenvectors $\matU$ constitute the basis for the Graph Fourier Transform (GFT) \citep{shuman2012gft}. The GFT of a graph signal $\bm{x} \in \reals^n$ is $\gft{\bm{x}} = \matU^\top \bm{x}$, and its inverse is $\bm{x} = \matU \gft{\bm{x}}$.

The task is transductive node regression, where the complete graph structure and the features $\matH^{(0)}$ for all $n$ nodes are accessible during training. However, the ground-truth labels, represented by a vector $\bm{y} \in \reals^n$, are only revealed for a subset of nodes. The node set $V$ is partitioned into a labeled training set $V_L$ of size $m$ and an unlabeled test set $V_U$ of size $u = n-m$.

A GNN is a parameterized function, denoted $f_{\bm{\theta}}$, that maps the graph structure and node features to a vector of predictions for all nodes, $\hat{\bm{y}} = f_{\bm{\theta}}(\matH^{(0)}, G) \in \reals^n$. The performance of the model is measured by a loss function $\ell(\cdot, \cdot)$. We distinguish between two key performance metrics:
\begin{itemize}[leftmargin=0.5cm]
    \setlength\itemsep{-2pt}
    \item The empirical risk, $R_L(f)$, is the average loss computed on the labeled training data. It measures how well the model fits the data it was trained on:
    \begin{equation}R_L(f) = \frac{1}{m} \sum_{i \in V_L} \ell(\hat{y}_i, y_i).\end{equation}
    \item The generalization error, $R_U(f)$, is the average loss on the unlabeled data. In the transductive setting, this serves as the proxy for the true risk and measures how well the model performs on unseen nodes:
    \begin{equation}R_U(f) = \frac{1}{u} \sum_{i \in V_U} \ell(\hat{y}_i, y_i).\end{equation}
\end{itemize}
\paragraph{Notation convention.}
We overload $f$ to denote both the predictor and its output vector on all $n$ nodes; when context is clear we drop the arguments and write $f\in\mathbb{R}^n$, with $f_i$ meaning the prediction at node $i$ (i.e., $(f(\matH^{(0)},G))_i$). In the Fourier domain we write $\widehat f=\matU^\top f$ and $\widehat f_i$ for its $i$-th coefficient.

\paragraph{The Generalization Gap and Transductive Complexity.}
The central objective of our theoretical analysis is to understand and bound the generalization gap, defined as the absolute difference $|R_U(f) - R_L(f)|$\citep{bartlett2002rademacher}.
The generalization gap is the critical indicator of a model's ability to learn underlying patterns versus simply memorizing the training data. A large gap signifies overfitting, where the model has learned specific patterns of the training set that do not hold for the rest of the graph. The primary driver of overfitting is the complexity of the function class $\mathcal{F} = \{f_{\bm{\theta}} | \bm{\theta} \in \Theta\}$ from which the model is selected. To analyze this in the transductive setting, we use a specialized tool: TRC, which measures a function class's ability to fit random labels on a random partition of all $n$ nodes.

\begin{definition}[Transductive Rademacher Complexity \citep{el2009transductive}]
\label{def:trc}
Let $\mathcal{F}$ be a class of real-valued functions over a domain of $n$ points, let $m$ be the number of labeled points, and let $p \in [0, 0.5]$. Let $\bm{\sigma}=(\sigma_{1},...,\sigma_{n})^{\top}$ be a vector of independent random variables, where each $\sigma_{i}$ takes the value $+1$ or $-1$ with probability p, and $0$ with probability $1-2p$. The Transductive Rademacher Complexity of $\mathcal{F}$ is defined as:
\begin{equation} \mathfrak{R}_{m,n}(\mathcal{F}) = \left(\frac{1}{m}+\frac{1}{n-m}\right) \cdot \mathbb{E}_{\bm{\sigma}}\left[\sup_{f\in\mathcal{F}} \sum_{i=1}^n \sigma_{i} f_i\right]. 
\end{equation}
\end{definition}

The relationship between the generalization gap and TRC is fundamental. For a loss function with values in a range of width $C$, the gap is bounded with high probability over the random partition of labeled nodes.
\begin{restatable}[Transductive Generalization Bound \citep{el2009transductive}]{theorem}{trcbound}
\label{thm:trc_bound}
With probability at least $1-\delta$, for any predictor $f \in \mathcal{F}$:
\begin{equation}
\begin{split}
    R_U(f) \le{}& R_L(f) + \mathfrak{R}_{m,n}(\mathcal{F}) \\
                & + C_1\frac{n\sqrt{\min\{m,u\}}}{mu} + C_2\sqrt{\frac{n}{mu}\ln(\frac{1}{\delta})},
\end{split}
\end{equation}
where $u=n-m$, and $C_1, C_2$ are absolute constants.
\end{restatable}

This theorem shows that the generalization gap is controlled by the TRC. While TRC provides the formal framework, our analysis will derive a bound on the closely related Full Transductive Gaussian Complexity (FTGC), which is more suitable for spectral analysis due to the properties of Gaussian variables.

\begin{definition}[FTGC]
\label{def:gc}
The Gaussian complexity of a function class $\mathcal{F}$ over the entire vertex set $V$ is:
\begin{equation} \mathcal{G}_V(\mathcal{F}) = \mathbb{E}_{\bm{g}} \left[ \sup_{f \in \mathcal{F}} \frac{1}{n} \sum_{i=1}^n g_i f_i \right], \end{equation}
where $g_1, \dots, g_n \sim \mathcal{N}(0, 1)$ are i.i.d. standard Gaussian variables.
\end{definition}

The two complexity measures are linked by the following standard result, which we prove in App. \ref{app:proof_of_trc_gc} for completeness.

\begin{restatable}[TRC Bound by FTGC]{lemma}{trcgc}
\label{lem:trc_gc}
The TRC $\mathfrak{R}_{m,n}(\mathcal{F})$ is upper-bounded by the full transductive Gaussian complexity $\mathcal{G}_V(\mathcal{F})$ as:
\begin{equation}\mathfrak{R}_{m,n}(\mathcal{F}) \le \frac{n^2}{mu} \sqrt{\frac{\pi}{2}} \cdot \mathcal{G}_V(\mathcal{F}),\end{equation}
where $u=n-m$. For simplicity in the main generalization bound, we can denote the constant $\sqrt{\pi/2}$ as $C_{gc}$.
\end{restatable}

By substituting the result of Lemma \ref{lem:trc_gc} into the main generalization bound of Theorem \ref{thm:trc_bound}, we arrive at a final bound expressed in terms of FTGC.

\begin{restatable}[Generalization Gap by FTGC]{theorem}{gcgap}
\label{thm:gc_gap}
With probability at least $1-\delta$, for any predictor $f \in \mathcal{F}$:
\begin{equation}
\begin{split}
    R_U(f) \le{}& R_L(f) + \frac{n^2 C_{gc}}{mu} \mathcal{G}_V(\mathcal{F}) \\
                & + C_1\frac{n\sqrt{\min\{m,u\}}}{mu} + C_2\sqrt{\frac{n}{mu}\ln(\frac{1}{\delta})},
\end{split}
\end{equation}
where $\mathcal{G}_V(\mathcal{F})$ is the FTGC, $u=n-m$, and $C_{gc}, C_1, C_2$ are absolute constants.
\end{restatable}

This final theorem establishes the path for our analysis. Our primary technical goal is to derive a data-dependent bound on $\mathcal{G}_V(\mathcal{F})$ for spectral GNNs. This bound can then be substituted into Theorem \ref{thm:gc_gap} to provide a generalization bound.
To achieve this, we shift the entire analysis from the spatial domain to the graph Fourier domain. We will show that FTGC is invariant under this transformation, which allows us to analyze the model in a domain where the complex graph convolution operator simplifies to an element-wise product. This change of basis is the key that enables us to separate the contributions of the graph structure, the chosen filter basis, and the learnable network parameters, leading to a more interpretable bound. 
\paragraph{Node classification.}
Our bounds extend to multi-class node classification. Let the predictor output logits $F\in\reals^{n\times C}$ and train with a bounded, $L_\ell$-Lipschitz surrogate (e.g., multi-class hinge, or cross-entropy with logits clipped to $[-B,B]$ so $\ell\in[0,C]$). By the vector-contraction inequality for Gaussian/Rademacher complexities, composing with $\ell$ only scales the FTGC term by $L_\ell$ \citep{ bartlett2002rademacher}. Hence, our lemmas and theorems hold after replacing the regression loss by $\ell$ and $R_L,R_U$ by their classification counterparts.

\section{A Unified Framework for Spectral GNNs}
\label{sec:unified_framework}

To analyze generalization for spectral GNNs, we adopt a unified formulation. Popular models such as ChebNet \citep{defferrard2016ChebyNet}, GCN \citep{kipf2017gcn}, and BernNet \citep{he2021bernnet} appear as special cases of a learnable polynomial filter. This formulation treats the polynomial basis as a choice of parametrization, which allows us to study all these models simultaneously and separate graph-dependent terms from learnable parameters.

We define a spectral GNN with $L$ layers for node regression. The input is the feature matrix $\matH^{(0)} \in \reals^{n \times d_0}$. The propagation rule for layer $l \in \{0, \dots, L-1\}$ is:
\begin{equation}
    \matH^{(l+1)} = \sigma \left( g_{\vectheta^{(l)}}(\normadj) \matH^{(l)} \matW^{(l)} \right),
    \label{eq:gnn_layer}
\end{equation}
where $\matW^{(l)} \in \reals^{d_l \times d_{l+1}}$ is a learnable weight matrix, $\sigma$ is a non-linear activation function, and $g_{\vectheta^{(l)}}(\normadj)$ is a graph filter. For efficiency and localization, the filter is a $K$-order polynomial of the normalized adjacency matrix:
\begin{equation}
    g_{\vectheta^{(l)}}(\normadj) = \sum_{k=0}^{K} \theta_{k}^{(l)} P_k(\normadj),
    \label{eq:poly_filter}
\end{equation}
where $\{P_k\}_{k=0}^K$ is a chosen polynomial basis and $\vectheta^{(l)} \in \reals^{K+1}$ are the learnable filter coefficients. The action of this filter is best understood in the Fourier domain, where it becomes a multiplier on the graph frequencies:
\begin{equation}
    g_{\vectheta^{(l)}}(\normadj) = \matU \left( \sum_{k=0}^{K} \theta_{k}^{(l)} P_k(\matLambda) \right) \matU^\top =\matU g_{\vectheta^{(l)}}(\matLambda) \matU^\top.
\end{equation}
The filter's frequency response can be expressed compactly using a matrix derived from the polynomial basis and the graph spectrum.

\begin{definition}[Generalized Vandermonde Matrix]
\label{def:gen_vand_matrix}
For a polynomial basis $\{P_k\}_{k=0}^K$ and normalized adjacency eigenvalues $\{\lambda_i\}_{i=1}^n$, the generalized Vandermonde matrix $\matVp \in \reals^{n \times (K+1)}$ has entries: $(\matVp)_{ik} = P_k(\lambda_i)$.
\end{definition}
The vector of filter responses is $\matVp \vectheta^{(l)}$, so $g_{\vectheta^{(l)}}(\matLambda) = \text{diag}(\matVp \vectheta^{(l)})$. The GNN layer from (\ref{eq:gnn_layer}) can then be written as:
\begin{equation}
    \matH^{(l+1)} = \sigma \left( \matU \, \text{diag}(\matVp \vectheta^{(l)}) \, \matU^\top \matH^{(l)} \matW^{(l)} \right).
    \label{eq:fourier_layer}
\end{equation}
This formulation separates the graph structure ($\matU, \matVp$) from the learnable parameters ($\vectheta^{(l)}, \matW^{(l)}$).

App. \ref{app:fullbasisdef} lists the exact polynomial bases (Chebyshev, Bernstein, Legendre, Monomial) and their properties.

\section{Main Results: Generalization and Stability of Spectral GNNs}
\label{sec:main}

\subsection{Analysis in the Graph Fourier Domain}

Our goal is to bound the generalization gap by analyzing the complexity of the function class $\mathcal{F} = \{f_{\bm{\theta}}(\matH^{(0)}, G) | \bm{\theta} \in \Theta\}$ that the GNN represents. We use the FTGC (\Cref{def:gc}). A key insight of our analysis is that this complexity is invariant to the GFT. This allows us to move the analysis from the spatial (node) domain to the simpler spectral domain. Let $\gft{\mathcal{F}} = \{\gft{f} | f \in \mathcal{F}\}$ be the function class in the Fourier domain, where $\gft{f} = \matU^\top f$.

\begin{restatable}[FTGC Invariance]{lemma}{gci}
\label{lem:gci}
The FTGC of the GNN function class is the same in the spatial and Fourier domains. That is, $\mathcal{G}_V(\mathcal{F}) = \mathcal{G}_V(\gft{\mathcal{F}})$.
\end{restatable}
% \begin{proof}
% The proof relies on the rotational invariance of the standard multivariate Gaussian distribution. A detailed proof is in the App. \ref{app:proof_of_gci}.
% \end{proof}

A detailed proof is in the App. \ref{app:proof_of_gci}. This lemma allows us to analyze the network's behavior in the Fourier domain, where the complex graph convolution becomes a simple element-wise multiplication by the filter's frequency response. To formalize this, let $\gft{\matH}^{(l)} = \matU^\top \matH^{(l)}$ be the node features at layer $l$ in the Fourier domain. We define a transformed activation function, $\tau(\bm{Z}) = \matU^\top \sigma(\matU \bm{Z})$, which encapsulates the change of basis. Using this notation, the propagation rule from (\ref{eq:fourier_layer}) simplifies to:
\begin{equation}
    \gft{\matH}^{(l+1)} = \tau \left( \text{diag}(\matVp \vectheta^{(l)}) \gft{\matH}^{(l)} \matW^{(l)} \right).
    \label{eq:fourier_prop_tau}
\end{equation}
For our analysis to proceed, we must ensure that this transformed activation function preserves the properties of the original activation $\sigma$. The following lemma confirms this.

\begin{restatable}[Lipschitz Preservation of Transformed Activation]{lemma}{lippres}
\label{lem:lip_pres}
Let $\sigma: \reals \to \reals$ be an $\alpha$-Lipschitz function applied element-wise. Then, the transformed activation $\tau(\bm{Z}) = \matU^\top \sigma(\matU \bm{Z})$ is also $\alpha$-Lipschitz with respect to the Frobenius norm $\|\cdot\|_F$.
\end{restatable}
% \begin{proof}
% The proof uses the unitary invariance of the Frobenius norm. 
% A detailed proof can be found in App. \ref{app:proof_of_lip_pres}.
% \end{proof}

A detailed proof can be found in App. \ref{app:proof_of_lip_pres}. With these tools, we have established a simpler, equivalent representation of the GNN in the Fourier domain. This forms the foundation for deriving our main generalization bound in the next section.

\subsection{A Data-Dependent Generalization Bound}
\label{sec:data_dependent_bound}

Building on our Fourier domain framework, we now derive our main result: a data-dependent generalization bound. This bound connects the GNN's generalization error to the spectral properties of the input data, offering a more detailed view than data-independent analyses \citep{morris2023wl, Scarselli2018VCDimension}. We begin by defining the spectral energy of the input features.

\begin{definition}[Input Signal Spectral Energy]
The spectral energy of the input feature matrix $\matH^{(0)} \in \reals^{n \times d_0}$ at the graph frequency $\lambda_i$ is the squared Frobenius norm of the $i$-th row of its GFT, $\gft{\matH}^{(0)} = \matU^\top \matH^{(0)}$:
\begin{equation} \mathcal{E}_0(\lambda_i) := \|(\gft{\matH}^{(0)})_i\|_2^2 = \|(\matU^\top \matH^{(0)})_i\|_2^2. \end{equation}
\end{definition}

This quantity measures how much of the input signal's ``energy'' is concentrated at each graph frequency. Using this, we can state our main theorem.

\begin{restatable}[Data-Dependent FTGC Bound]{theorem}{datadependentbound}
\label{thm:data_dependent_bound}
Let $f_{\bm{\theta}}$ be an $L$-layer spectral GNN with an $\alpha$-Lipschitz activation $\sigma$ satisfying $\sigma(0)=0$. Assume the model parameters are constrained such that for each layer $l \in \{0, \dots, L-1\}$, the weight matrix norm $\|\matW^{(l)}\|_2 \le C_{W,l}$ and the filter coefficient norm $\|\vectheta^{(l)}\|_2 \le C_{\theta,l}$. The FTGC of the function class $\mathcal{F}$ is bounded by:
\begin{equation}
\begin{split}
    \mathcal{G}_V(\mathcal{F}) = \mathcal{G}_V(\gft{\mathcal{F}})
    \le {}&\frac{1}{\sqrt{n}} \|\matVp\|_{2,\infty}^{L-1} \left( \prod_{l=0}^{L-1} \alpha C_{W,l} C_{\theta,l} \right) \\
    \cdot {}& \left( \sum_{i=1}^n \|\bm{v}_i\|_2^2 \mathcal{E}_0(\lambda_i) \right)^{1/2},
\end{split}
\label{eq:data_dependent_bound}
\end{equation}
where $\bm{v}_i$ is the $i$-th row of the generalized Vandermonde matrix $\matVp$, and $\|\matVp\|_{2,\infty} = \max_i \|\bm{v}_i\|_2$ is the maximum row-norm of $\matVp$.
\end{restatable}

% \begin{proof}
% The full proof is provided in App. \ref{app:proof_of_thm3}. The core strategy is to first bound the FTGC by the Frobenius norm of the final layer's features in the Fourier domain, $\|\gft{\matH}^{(L)}\|_F$. We then unroll the network's recursive propagation rule. For the first layer, we derive a tight, data-dependent bound that captures the interaction between the input signal's spectrum and the filter basis. For all subsequent layers, we use a data-independent worst-case bound based on the maximum amplification of the filter basis. Combining these recursive bounds yields the final result.
% \end{proof}

The full proof is provided in App. \ref{app:proof_of_thm3}. This bound reveals how different factors contribute to the model complexity. The term $\prod_{l=1}^{L-1}(\cdot)$ shows an exponential dependency on the network depth for layers $L-1$ down to $1$. Crucially, the final term, $\left( \sum_{i=1}^n \|\bm{v}_i\|_2^2 \mathcal{E}_0(\lambda_i) \right)^{1/2}$, captures the interaction between the data and the model at the first layer. Here, $\|\bm{v}_i\|_2^2$ measures the potential amplification of the filter basis at frequency $\lambda_i$, while $\mathcal{E}_0(\lambda_i)$ is the signal's energy at that frequency. The bound is minimized when the signal's energy is concentrated at frequencies where the filter basis has a small response magnitude. This provides a key insight: good generalization is associated with spectral filters that avoid excessively amplifying the dominant frequencies of the input signal.

Next, we bound the FTGC for the case when we have no nonlinearities. By avoiding the initial bounding of the complexity by the output norm, we can derive a result that is tighter than the general non-linear bound from Theorem \ref{thm:data_dependent_bound} and which more clearly reveals the influence of network depth on the model's complexity.

We consider an $L$-layer GNN with the non-linear activation $\sigma$ removed, so the propagation rule becomes $\matH^{(l+1)} = g_{\vectheta^{(l)}}(\normadj) \matH^{(l)} \matW^{(l)}$. 

\begin{restatable}[Tighter Complexity Bound for Deep Linear GNNs]{theorem}{tighterbounddeep}
\label{thm:tighter_bound_deep}
Let $f_{\bm{\theta}}$ be an $L$-layer linear spectral GNN with a single output feature. Assume the model parameters are constrained such that $\|\matW^{(l)}\|_2 \le C_{W,l}$, and $\|\vectheta^{(l)}\|_2 \le C_{\theta,l}$ for all layers $l$. The FTGC of the function class $\mathcal{F}$ is bounded by:
\begin{equation} \mathcal{G}_V(\mathcal{F}) \le \frac{1}{n} \left( \prod_{l=0}^{L-1} C_{W,l} C_{\theta,l} \right) \left( \sum_{i=1}^n \|\bm{v}_i\|_2^{2L} \mathcal{E}_0(\lambda_i) \right)^{1/2}, \end{equation}
where $\bm{v}_i$ is the $i$-th row of $\matVp$ and $\mathcal{E}_0(\lambda_i)$ is the input signal spectral energy.
\end{restatable}

% \begin{proof}
% The proof proceeds by sequentially taking the supremum over the parameter constraints, starting from the final layer, and then bounding the expectation of the resulting random vector's norm.
% \end{proof}

The full proof is in App. \ref{app:proof_of_thm_linear}. This bound provides a more precise characterization of complexity for deep linear models. Compared to the general non-linear bound in (\ref{eq:data_dependent_bound}), it is tighter by a factor of $1/\sqrt{n}$. Most importantly, it clearly shows how the potential for complexity grows with depth. The data-dependent term now contains the factor $\|\bm{v}_i\|_2^{2L}$, indicating that any spectral instability in the polynomial basis (large $\|\bm{v}_i\|_2$ for some frequency $\lambda_i$) is amplified exponentially with the number of layers $L$. This provides a good theoretical justification for choosing spectrally stable polynomial bases (\textit{e.g.}, Chebyshev or Bernstein) when designing deep spectral GNNs, even in the absence of non-linearities.
In App. \ref{app:boundvsgap} we plot our FTGC-based bound from Theorem \ref{thm:data_dependent_bound} against the measured generalization gap across polynomial order, bases, and datasets.

\subsection{Worst-Case Stability via Jacobian Norm}

Our analysis reveals that the same core principles governing the generalization error of a spectral GNN also control its stability. This property, which measures the sensitivity of a model's output to small perturbations in its input, is important for ensuring robustness and preventing issues like exploding gradients during training \citep{bousquet2002stability}. In this section, we formalize this connection by analyzing the network's Jacobian from a worst-case perspective.

A standard way to quantify stability is by bounding the spectral norm of the network's Jacobian \citep{hariri2025returnchebnet, yoshida2017spectral, novak2018sensitivity}. A large Jacobian norm implies that small input changes can be amplified into large output changes. We define the network Jacobian as $\mathcal{J} = \frac{\partial \text{vec}(\matH^{(L)})}{\partial \text{vec}(\matH^{(0)})}$, where $\text{vec}(\cdot)$ vectorizes the feature matrices.

\begin{restatable}[Jacobian Norm Bound]{theorem}{jacobianbound}
\label{thm:jacobian_bound}
Let the GNN be an $L$-layer spectral GNN with an $\alpha$-Lipschitz and continuously differentiable activation $\sigma$ satisfying $\sigma(0)=0$. Under the same parameter constraints as Theorem \ref{thm:data_dependent_bound}, the spectral norm of the network Jacobian $\mathcal{J}$ is bounded by:
\begin{equation} \|\mathcal{J}\|_2 \le \prod_{l=0}^{L-1} \left( \alpha C_{W,l} C_{\theta,l} \|\matVp\|_{2,\infty} \right), \end{equation}
where $\|\matVp\|_{2,\infty} = \max_{i} \|\bm{v}_i\|_2$ is the maximum row-norm of the generalized Vandermonde matrix.
\end{restatable}

% \begin{proof}
% The full proof is in Appendix \ref{app:proof_of_thm4}. The analysis is performed in the Fourier domain, where the Jacobian's norm is identical to that in the spatial domain due to unitary invariance. The chain rule decomposes the total Jacobian into a product of per-layer Jacobians, and we bound the norm of each to arrive at the result.
% \end{proof}

The full proof is in App. \ref{app:proof_of_thm4}. This bound quantifies the worst-case sensitivity of the GNN. It reveals an exponential dependency on depth ($L$) and highlights the central role of $\|\matVp\|_{2,\infty}$, the maximum amplification potential of the polynomial basis. We empirically validate this bound in App.~\ref{app:jacobian_tightness}, demonstrating that it bounds the true spectral norm of the network Jacobian within a small constant factor across varying polynomial orders.

\subsection{A Unified Principle for Architectural Design }
\label{sec:unified_principle_design}

Comparing our generalization bound (Theorem \ref{thm:data_dependent_bound}) and our stability bound (Theorem \ref{thm:jacobian_bound}) reveals a unified principle: the same architectural factors govern both properties. The core components—the activation's Lipschitz constant ($\alpha$), the norms of the weights ($C_{W,l}, C_{\theta,l}$), and the characteristics of the polynomial basis appear in both bounds.

The worst-case stability is controlled by the maximum amplification factor of the basis, $\|\matVp\|_{2,\infty}$. This term also appears in the generalization bound for the deeper layers. However, the generalization bound is more nuanced. For the first layer, it depends on the fine-grained interaction term $\left( \sum_{i=1}^n \|\bm{v}_i\|_2^2 \mathcal{E}_0(\lambda_i) \right)^{1/2}$, which measures how the filter basis amplifies the input signal's actual spectral energy distribution.

This provides a clear design strategy. To ensure a model is both stable and generalizable, one should start by selecting a polynomial basis with a low maximum amplification, $\|\matVp\|_{2,\infty}$. A large $\|\bm{v}_i\|_2$ means the chosen polynomial basis is highly sensitive to that frequency, so any signal energy present there will contribute more significantly to the model's complexity. Good generalization is therefore achieved when the signal's energy is concentrated at frequencies the filter basis considers less important.

This spectral perspective offers a distinct advantage over analyses in the spatial domain \citep{Esser2021LearningTheory}, which often rely on the diffusion operator's norm $\|\normadj\|_{\infty}$ (bounded by $O(\sqrt{degree_{max}/degree_{min}})$) or on the maximum degree \citep{Liao2021PACBayesian}.  As we empirically demonstrate in App.~\ref{app:bound_comparison}, this reliance causes spatial bounds to grow exponentially with network depth. Our approach effectively trades a dependency on the graph's structural irregularity for a dependency on the mathematical stability of the filter basis. For many common architectures like GCN, this trade-off is highly favorable, as a simple polynomial basis can yield a small, constant $\|\matVp\|_{2,\infty}$ (\textit{e.g.}, $1$ for GCN) regardless of the graph's degree disparity, leading to a much tighter guarantee.
By selecting a spectrally stable polynomial basis (low $\|\matVp\|_{2,\infty}$) and designing filters that do not amplify the dominant frequencies of the input signal, we can simultaneously build models that are less prone to overfitting and exploding gradients.

\section{Design Implications: Basis Stability and Oversmoothing}\label{sec:design_implications}

Our theoretical bounds in Section \ref{sec:main} show that the generalization and stability of a spectral GNN depend on the properties of the chosen polynomial basis $\{P_k\}$. This dependency appears through the term $\|\bm{v}_i\|^2_{2}$, which is the squared row norm of the generalized Vandermonde matrix $\matVp$. 

\begin{figure}[t]
\centering
\includegraphics[width=1\linewidth]{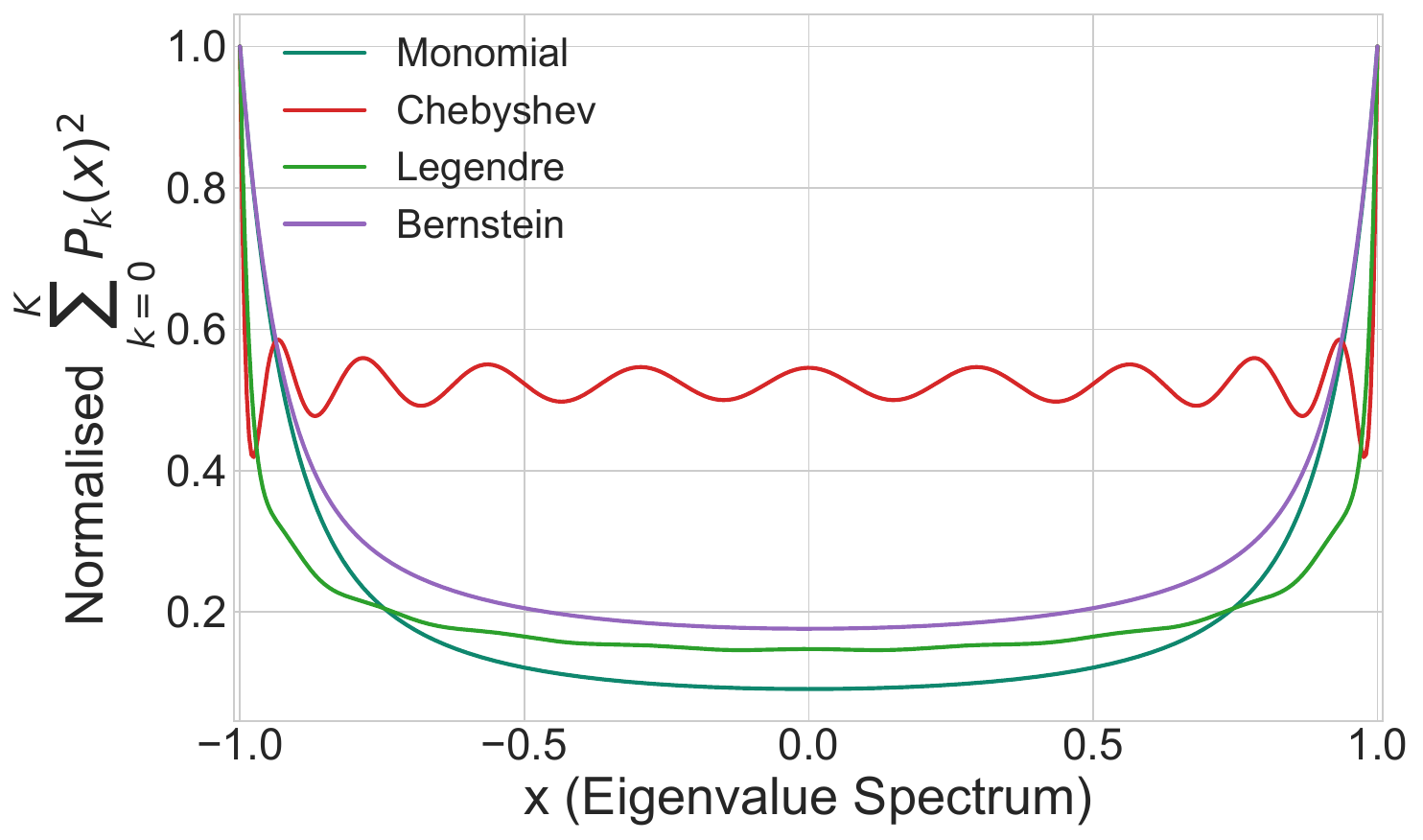}
\caption{Basis stability across the spectrum ($K=10$).}
\label{fig:basis_stability}
\end{figure}

We refer to the function $x\mapsto \sum_{k=0}^{K}P_k(x)^2$ as the amplification profile of the basis, as it shows how sensitive the basis is to different frequencies. When this function is evaluated at the graph eigenvalues $\lambda_i$, its value is exactly equal to $\|\bm{v}_i\|^2_{2}$.
In \Cref{fig:basis_stability}, we plot the amplification profiles for different bases. From the plot, we can group these bases, which are normalized to a maximum value of 1, into three classes based on their {uniformity}:
\begin{enumerate}[leftmargin=0.5cm]
    \setlength\itemsep{-2pt}
    \item The \textit{Chebyshev} basis is the most uniform. Its profile is relatively flat across the middle of the spectrum, with its minimum value (around 0.5) being comparatively close to its maximum value of 1.0.
    
    \item The \textit{Bernstein} and \textit{Legendre} bases have a clear U-shape, making them non-uniform. Their sensitivity is much higher at the spectral edges than in the center.
    
    \item The \textit{Monomial} basis is the least uniform. It exhibits the most extreme U-shape, with the largest comparative difference between its sensitivity at the edges versus the center.
\end{enumerate}

This uniformity has direct consequences for generalization. As shown in our main bound (\Cref{thm:data_dependent_bound}), model complexity depends on the interaction term $\sum_{i} \|\bm{v}_i\|_2^2 \mathcal{E}_0(\lambda_i)$. For a non-uniform basis, the model's complexity becomes highly sensitive to the signal's energy distribution ($\mathcal{E}_0(\lambda_i)$), leading to unpredictable performance. The uniform profile of Chebyshev makes the complexity less dependent on the signal, which should lead to more consistent performance across different graphs.

\begin{table*}[t]
\centering
\caption{\textbf{Test accuracy improvement depends on basis stability.} The column $\Delta$ uses empirically calculated $\rho$  for paired significance ($^*p<0.05, ^{**}p<0.01, ^{***}p<0.001$). While other bases occasionally degrade performance (negative $\Delta$), a consistent positive improvement is observed across all datasets only for the Chebyshev basis.}
\label{tab:accuracy}
\resizebox{\textwidth}{!}{%
\begin{tabular}{l ccc ccc ccc ccc}
\toprule
& \multicolumn{3}{c}{\textbf{Chameleon}} & \multicolumn{3}{c}{\textbf{Squirrel}} & \multicolumn{3}{c}{\textbf{Cora}} & \multicolumn{3}{c}{\textbf{Citeseer}} \\
\cmidrule(lr){2-4} \cmidrule(lr){5-7} \cmidrule(lr){8-10} \cmidrule(lr){11-13}
\textbf{Basis} & \textbf{Base} & \textbf{+Reg} & \textbf{$\Delta$} & \textbf{Base} & \textbf{+Reg} & \textbf{$\Delta$} & \textbf{Base} & \textbf{+Reg} & \textbf{$\Delta$} & \textbf{Base} & \textbf{+Reg} & \textbf{$\Delta$} \\
\midrule
Monomial & $43.68 \pm 2.0$ & $44.28 \pm 1.2$ & $+0.60\uparrow$ & $27.87 \pm 1.2$ & $26.76 \pm 0.8$ & $-1.11^{*}\downarrow$ & $78.33 \pm 1.4$ & $75.34 \pm 1.5$ & $ -2.99^{***}\downarrow$ & $66.15 \pm 1.9$ & $64.30 \pm 2.5$ & $-1.85\downarrow$ \\
Legendre & $41.22 \pm 1.8$ & $42.74 \pm 1.9$ & $+1.52\uparrow$ & $26.32 \pm 1.7$ & $29.15 \pm 1.0$ & $+2.83^{***}\uparrow$ & $79.53 \pm 1.4$ & $79.36 \pm 1.4$ & $-0.17\downarrow$ & $65.35 \pm 1.7$ & $66.77 \pm 2.0$ & $+1.42\uparrow$ \\
Bernstein & $40.58 \pm 1.8$ & $42.23 \pm 1.9$ & $+1.65^{*}\uparrow$ & $26.44 \pm 0.6$ & $31.23 \pm 0.9$ & $+4.79^{***}\uparrow$ & $78.45 \pm 2.0$ & $79.08 \pm 1.7$ & $+0.63\uparrow$ & $65.47 \pm 2.0$ & $65.30 \pm 2.0$ & $-0.17\downarrow$ \\
Chebyshev & {$41.76 \pm 1.7$} & {$43.29 \pm 1.9$} & {$+1.53\uparrow$} & {$27.89 \pm 0.8$} & {$31.20 \pm 0.7$} & {$+3.31^{***}\uparrow$} & {$77.64 \pm 2.0$} & {$78.23 \pm 2.1$} & {$+0.59\uparrow$} & {$65.57 \pm 1.6$} & {$65.88 \pm 1.6$} & {$+0.31\uparrow$} \\
\bottomrule
\end{tabular}
}
\end{table*}

\begin{table*}[t]
\centering
\caption{\textbf{The regularizer's effect on the generalization gap.}  Column $\Delta$ shows the change, where a negative value indicates a reduction in the gap ($^*p<0.05, ^{**}p<0.01, ^{***}p<0.001$). The gap reduction is most consistent and significant on the Chebyshev basis.}
\label{tab:gap}
\resizebox{\textwidth}{!}{%
\begin{tabular}{l ccc ccc ccc ccc}
\toprule
& \multicolumn{3}{c}{\textbf{Chameleon}} & \multicolumn{3}{c}{\textbf{Squirrel}} & \multicolumn{3}{c}{\textbf{Cora}} & \multicolumn{3}{c}{\textbf{Citeseer}} \\
\cmidrule(lr){2-4} \cmidrule(lr){5-7} \cmidrule(lr){8-10} \cmidrule(lr){11-13}
\textbf{Basis} & \textbf{Base} & \textbf{+Reg} & \textbf{$\Delta$} & \textbf{Base} & \textbf{+Reg} & \textbf{$\Delta$} & \textbf{Base} & \textbf{+Reg} & \textbf{$\Delta$} & \textbf{Base} & \textbf{+Reg} & \textbf{$\Delta$} \\
\midrule
Monomial & $4.93 \pm 0.97$ & $1.99 \pm 0.17$ & $-2.94^{***}\downarrow$ & $4.21 \pm 1.43$ & $0.89 \pm 0.3$ & $-3.32^{***}\downarrow$ & $0.75 \pm 0.06$ & $0.63 \pm 0.04$ & $-0.12^{***}\downarrow$ & $1.02 \pm 0.04$ & $1.02 \pm 0.05$ & $+0.00$ \\
Legendre & $4.72 \pm 1.1$ & $2.23 \pm 0.35$ & $-2.49^{***}\downarrow$ & $2.04 \pm 0.66$ & $1.79 \pm 0.07$ & $-0.25\downarrow$ & $0.65 \pm 0.06$ & $0.49 \pm 0.04$ & $-0.16^{***}\downarrow$ & $0.83 \pm 0.06$ & $0.90 \pm 0.03$ & $+0.07^{**}\uparrow$ \\
Bernstein & $2.14 \pm 0.3$ & $3.46 \pm 1.38$ & $+1.32^{*}\uparrow$ & $2.44 \pm 0.55$ & $0.29 \pm 0.18$ & $-2.15^{***}\downarrow$ & $0.64 \pm 0.06$ & $0.54 \pm 0.04$ & $-0.10^{***}\downarrow$ & $1.00 \pm 0.05$ & $0.99 \pm 0.03$ & $-0.01\downarrow$ \\
Chebyshev & {$3.06 \pm 1.0$} & {$2.20 \pm 0.17$} & {$-0.86\downarrow$} & {$2.77 \pm 0.31$} & {$0.09 \pm 0.09$} & {$-2.68^{***}\downarrow$} & {$0.79 \pm 0.07$} & {$0.69 \pm 0.06$} & {$-0.10^{**}\downarrow$} & {$0.99 \pm 0.02$} & {$0.92 \pm 0.06$} & {$-0.07^{*}\downarrow$} \\
\bottomrule
\end{tabular}
}
\end{table*}

Furthermore, this analysis provides insight into the problem of \textit{oversmoothing} \citep{li2018deeper} in deep GNNs. Our bound for deep linear models (\Cref{thm:tighter_bound_deep}) includes the term $\|\bm{v}_i\|_2^{2L}$. As depth $L$ increases, the model's complexity becomes exponentially dominated by the largest values of the amplification profile. For non-uniform bases, this forces the model to focus only on the spectral edges, which is a potential cause of oversmoothing \citep{nt2019revisitinggraphneuralnetworks, oono2020graph}.

\textbf{Practical regularization.}
The first-layer, data-dependent term of our nonlinear bound (Eq. \ref{eq:data_dependent_bound}) suggests penalizing energy-weighted (EW) amplification.
To discourage large gains at frequencies where the input has energy, we use the simple ratio:
\begin{equation}\label{eq:ew_ratio_single}
\mathcal{R}_{\mathrm{EW}}
\;=\;
\frac{\|\mathrm{diag}\!(\matV_{P}\vectheta)\,\gft{\matH}^{(0)}\|_F^{2}}{\|\gft{\matH}^{(0)}\|_F^{2}}
\;=\;
\frac{\|g_{\vectheta}(\normadj) \matH^{(0)}\|_F^{2}}
{\|\matH^{(0)}\|_F^2}.
\end{equation}

In practice, we detach $\matH^{(0)}$ in this penalty term so that the gradient only updates the filter parameters. Training details are provided in Section~\ref{sec:empirical_validation}.

\section{Empirical Validation}
\label{sec:empirical_validation}

\subsection{Experimental Setup}

To test the hypothesis from \Cref{sec:design_implications}, we conduct node classification experiments on several benchmark datasets using the different polynomial bases.

\textbf{Model.} The core of our experimental model consists of a single linear spectral graph layer, as defined in \Cref{sec:unified_framework}. A residual connection is added to the graph layer to improve training stability. To enhance expressive power, the input node features are first processed by an MLP before being passed to the graph layer. The output features are then passed to a linear classifier for node classification. Even with these components, our central graph convolution layer remains simple and linear. This allows us to isolate and observe the effects of the different polynomial bases and our spectral regularizer, which is the primary goal of this experiment.

\textbf{Regularization.} We apply standard dropout and add the regularization term $\lambda_{\text{EW}}\mathcal{R}_{\mathrm{EW}}$ from Eq. (\ref{eq:ew_ratio_single}), with $\lambda_{\text{EW}}$ tuned on the validation set.
\textbf{Evaluation.} We evaluate all models on four representative datasets. For each dataset we use 10 random sparsified splits: 10 labeled nodes per class for training; from the remaining nodes we use 35\% for validation and 65\% for test. Results are reported as mean $\pm$ 95\% confidence interval (CI) over the 10 splits. We report two metrics: test accuracy (\%) and the generalization gap. We define the generalization gap as testing loss minus training loss. For both tables we do two hyperparameter searches before and after adding the regularisation. The full details are in App. \ref{app:hyper}.

\subsection{Results and Analysis}

The results of our experiments are presented in \Cref{tab:accuracy} for test accuracy and \Cref{tab:gap} for the generalization gap. They confirm our analysis from \Cref{sec:design_implications}.

First, for the \textit{highly non-uniform Monomial basis}, the regularizer's effect is poor. As shown in \Cref{tab:accuracy}, it fails to improve test accuracy and often hurts performance significantly (\textit{e.g.}, on Cora and Citeseer).
Second, for the \textit{moderately non-uniform Bernstein and Legendre bases}, the results are mixed. The regularizer sometimes improves accuracy (\textit{e.g.}, Bernstein on Squirrel), but the effect is inconsistent across datasets. For example, Bernstein's performance decreases on Citeseer, and Legendre's effect on Cora is negligible. The gap reduction for these bases, shown in \Cref{tab:gap}, is also unreliable.
Finally, for the \textit{uniform Chebyshev basis}, the results are consistently positive. \Cref{tab:gap} shows that the regularizer reliably reduces the generalization gap, and \Cref{tab:accuracy} shows that this translates directly into a consistent improvement in test accuracy across all datasets.

These results provide clear evidence that the uniformity of a basis's amplification profile is a key predictor of its response to spectrally-aware regularization. The predictable behavior of the uniform Chebyshev basis allows the regularizer to function effectively and consistently.

\section{Limitations}

Our analysis targets the transductive, fixed-graph setting and does not cover inductive generalization to new nodes or graphs, or distribution shift. We study polynomial spectral filters on the normalized adjacency and assume $\alpha$-Lipschitz activations with bounded parameter norms. Attention layers, non-polynomial operators, and other normalizations are outside our scope. Furthermore, while we provide a sharper bound for deep linear networks, our non-linear bound relies on sequential Lipschitz contractions. Applying covering number arguments and Dudley's entropy integral could yield tighter capacity bounds for the non-linear case. Finally, our derived bounds help us understand spectral GNNs, but they are not direct predictors of test accuracy. They explain trends rather than provide point estimates. 

\section{Conclusions}

We presented a unified Fourier-domain analysis of spectral GNNs that leads to new, depth and order-aware generalization bounds.
By showing that FTGC is invariant under the GFT, we derived data-dependent bounds that explicitly capture the interaction between polynomial bases, filter amplification, and the spectral energy of the input.
Our analysis also established tighter results in the linear setting and revealed that the same factors governing generalization also control network stability. 
These insights yield a simple regularizer and a clear design principle: spectrally stable bases and filters that avoid amplifying dominant frequencies improve both generalization and robustness.
Together, our theoretical and empirical findings provide a principled foundation for understanding and guiding the design of spectral GNN architectures.

\section*{Acknowledgements}
Vahan Martirosyan is the recipient of a PhD scholarship from the STIC Doctoral School of Université Paris-Saclay.

\bibliography{aistats_conference} % You would need a references.bib file
%\clearpage
\section*{Checklist}

\begin{enumerate}

  \item For all models and algorithms presented, check if you include:
  \begin{enumerate}
    \item A clear description of the mathematical setting, assumptions, algorithm, and/or model. [Yes — Mathematical setting, assumptions, and model are in \Cref{sec:unified_framework} and \Cref{sec:main}, with notation in \Cref{sec:preliminaries}.]
    \item An analysis of the properties and complexity (time, space, sample size) of any algorithm. [Not Applicable — We do not introduce a new algorithm.]
    \item (Optional) Anonymized source code, with specification of all dependencies, including external libraries. [Yes — The link to the anonymous repository for the code is provided in the Appendix.]
  \end{enumerate}

  \item For any theoretical claim, check if you include:
  \begin{enumerate}
    \item Statements of the full set of assumptions of all theoretical results. [Yes — Assumptions ($\alpha$-Lipschitz $\sigma$, bounded norms) are stated before each result (\Cref{sec:main}).]
    \item Complete proofs of all theoretical results. [Yes — Complete proofs are provided in the Appendix.]
    \item Clear explanations of any assumptions. [Yes — We explain assumptions before theorems and lemmas (\Cref{sec:main}).]     
  \end{enumerate}

  \item For all figures and tables that present empirical results, check if you include:
  \begin{enumerate}
    \item The code, data, and instructions needed to reproduce the main experimental results (either in the supplemental material or as a URL). [Yes — The anonymous repository hosting the code to reproduce the results of this paper is provided in the Appendix.]
    \item All the training details (e.g., data splits, hyperparameters, how they were chosen). [Yes — We provide full details regarding training in the Appendix.]
    \item A clear definition of the specific measure or statistics and error bars (e.g., with respect to the random seed after running experiments multiple times). [Yes — All experiments are repeated across 10 random  splits, with further details reported in Section \ref{sec:empirical_validation}.]
    \item A description of the computing infrastructure used. (e.g., type of GPUs, internal cluster, or cloud provider). [Yes — We report complete details about the computing infrastructure in the Appendix.]
  \end{enumerate}

  \item If you are using existing assets (e.g., code, data, models) or curating/releasing new assets, check if you include:
  \begin{enumerate}
    \item Citations of the creator If your work uses existing assets. [Yes — All the needed references and credits have been explicitly mentioned in our code.]
    \item The license information of the assets, if applicable. [Yes — All license information of the existing assets were properly included.]
    \item New assets either in the supplemental material or as a URL, if applicable. [Yes — The only new assets of our work are the implemented codes for which we report a detailed Readme file in the anonymous repository.]
    \item Information about consent from data providers/curators. [Not Applicable — Since all used datasets are open-access, no permission was needed.]
    \item Discussion of sensible content if applicable, e.g., personally identifiable information or offensive content. [Not Applicable — Since all used datasets are open-access, no permission was needed.]
  \end{enumerate}

  \item If you used crowdsourcing or conducted research with human subjects, check if you include:
  \begin{enumerate}
    \item The full text of instructions given to participants and screenshots. [Not Applicable — Our paper does not involve crowdsourcing nor research with human subjects.]
    \item Descriptions of potential participant risks, with links to Institutional Review Board (IRB) approvals if applicable. [Not Applicable — Our paper does not involve crowdsourcing nor research with human subjects.]
    \item The estimated hourly wage paid to participants and the total amount spent on participant compensation. [Not Applicable — Our paper does not involve crowdsourcing nor research with human subjects.]
  \end{enumerate}

\end{enumerate}

\clearpage
\appendix
\thispagestyle{empty}

\onecolumn

\aistatstitle{Supplementary Materials}

\section*{Appendix Organization and Notation}

This appendix provides proofs for all theoretical claims made in the main paper, presents further empirical analyses that validate our theory, and includes additional details to ensure reproducibility.
\begin{itemize}[leftmargin=0.5cm]
    \setlength\itemsep{-2pt}
    \item \textbf{Section Norms and Products:} We define the vector and matrix norms used throughout the analysis.
    \item \textbf{Section \ref{app:proof_of_trc_gc}:} Proof of Lemma \ref{lem:trc_gc}, relating TRC and FTGC.
    \item \textbf{Section \ref{app:proof_of_gci}:} Proof of Lemma \ref{lem:gci}, showing FTGC invariance under the GFT.
    \item \textbf{Section \ref{app:proof_of_lip_pres}:} Proof of Lemma \ref{lem:lip_pres}, on the Lipschitz preservation of the transformed activation.
    \item \textbf{Section \ref{app:proof_of_thm3}:} Proof of Theorem \ref{thm:data_dependent_bound}, our main data-dependent FTGC bound.
    \item \textbf{Section \ref{app:proof_of_thm_linear}:} Proof of Theorem \ref{thm:tighter_bound_deep}, the tighter FTGC bound for linear models.
    \item \textbf{Section \ref{app:proof_of_thm4}:} Proof of Theorem \ref{thm:jacobian_bound}, the network Jacobian norm bound.
    \item \textbf{Section \ref{app:jacobian_tightness}:} Empirical validation demonstrating the tightness of our Jacobian norm bound.
    \item \textbf{Section \ref{app:boundvsgap}:} Empirical validation of our theoretical bound against the measured generalization gap, including sensitivity analysis and large-scale validation on ogbn-arxiv dataset.
    \item \textbf{Section \ref{app:bound_comparison}:} Quantitative comparison of our FTGC bound against prior spatial transductive bounds.
    \item \textbf{Section \ref{app:fullbasisdef}:} Formal definitions and properties of the polynomial bases discussed.
    \item \textbf{Section \ref{app:hyper}:} Full details on experimental setup, including hyperparameters, datasets, and splits.
\end{itemize}

\subsection*{Norms and Products}
\label{app:notation}
We use the following standard norms and matrix products in our analysis:
\begin{itemize}[leftmargin=0.5cm]
    \setlength\itemsep{-2pt}
    \item \textbf{Vector $l_2$-norm:} For a vector $\bm{x} \in \reals^n$, $\|\bm{x}\|_2 = \sqrt{\sum_{i=1}^n x_i^2}$.
    \item \textbf{Matrix Frobenius norm:} For a matrix $\bm{A} \in \reals^{m \times n}$, the Frobenius norm is $\|\bm{A}\|_F = \sqrt{\sum_{i=1}^m \sum_{j=1}^n A_{ij}^2} = \sqrt{\text{Tr}(\bm{A}^\top \bm{A})}$.
    \item \textbf{Matrix spectral norm:} For a matrix $\bm{A} \in \reals^{m \times n}$, the spectral norm (or operator norm) is induced by the vector $l_2$-norm: $\|\bm{A}\|_2 = \sup_{\|\bm{x}\|_2=1} \|\bm{A}\bm{x}\|_2 = \sigma_{\max}(\bm{A})$, where $\sigma_{\max}(\bm{A})$ is the largest singular value of $\bm{A}$.
    \item \textbf{Matrix $(2,\infty)$-norm:} For a matrix $\bm{A} \in \reals^{m \times n}$ with rows $\bm{a}_1, \dots, \bm{a}_m$, this norm is the maximum $l_2$-norm among all its rows: $\|\bm{A}\|_{2,\infty} = \max_{i=1, \dots, m} \|\bm{a}_i\|_2$.
    \item \textbf{Kronecker product:} For matrices $\bm{A} \in \reals^{m \times n}$ and $\bm{B} \in \reals^{p \times q}$, their Kronecker product $\bm{A} \otimes \bm{B}$ is an $mp \times nq$ block matrix where the $(i,j)$-th block is the $p \times q$ matrix $A_{ij}\bm{B}$.
    \item \textbf{Khatri-Rao product:} For matrices $\bm{A} \in \reals^{m \times k}$ and $\bm{B} \in \reals^{n \times k}$ with the same number of columns, their Khatri-Rao (or column-wise Kronecker) product $\bm{A} \odot \bm{B}$ is an $mn \times k$ matrix where the $j$-th column is the Kronecker product of the $j$-th columns of $\bm{A}$ and $\bm{B}$, i.e., $(\bm{A} \odot \bm{B})_{:,j} = \bm{A}_{:,j} \otimes \bm{B}_{:,j}$. In our proofs, we use a row-wise variant, denoted $\odot_r$.
\end{itemize}

\section{Proof of Lemma \ref{lem:trc_gc} (Relation Between TRC and FTGC)}
\label{app:proof_of_trc_gc}

\trcgc* 

\begin{proof}
The proof connects the two complexity measures by showing that the expectation over the TRC random variables ($\sigma_i$) is bounded by the expectation over standard Gaussian variables ($g_i$). This is done by using standard Rademacher variables ($\epsilon_i \in \{-1, +1\}$) as an intermediate step.

We begin by defining the unnormalized core expectation terms:
\begin{align*}
    E_{TRC} &= \mathbb{E}_{\bm{\sigma}}\left[\sup_{f\in\mathcal{F}} \sum_{i=1}^n \sigma_{i} f_i\right], \\
    E_{GC} &= \mathbb{E}_{\bm{g}}\left[\sup_{f\in\mathcal{F}} \sum_{i=1}^n g_{i} f_i\right],
\end{align*}
where $\bm{\sigma}$ is the TRC random vector and $\bm{g}$ is a vector of i.i.d. $\mathcal{N}(0, 1)$ variables.

A random variable $\sigma_i$ from the TRC definition (value $+1$ or $-1$ with probability $p$, and $0$ with probability $1-2p$) can be constructed as the product $\sigma_i = \delta_i \epsilon_i$. Here, $\epsilon_i$ is a standard Rademacher variable (taking values $+1$ or $-1$ with probability 0.5), and $\delta_i$ is an independent Bernoulli variable (taking value $1$ with probability $2p$ and $0$ otherwise).

Using this construction, we can rewrite $E_{TRC}$ and apply the law of total expectation:
\begin{equation}E_{TRC} = \mathbb{E}_{\bm{\delta}, \bm{\epsilon}}\left[\sup_{f\in\mathcal{F}} \sum_{i=1}^n \delta_i \epsilon_i f_i\right] = \mathbb{E}_{\bm{\delta}}\left[\mathbb{E}_{\bm{\epsilon}}\left[\sup_{f\in\mathcal{F}} \sum_{i=1}^n \epsilon_i (\delta_i f_i)\right]\right].\end{equation}
We now apply the contraction principle of \citet{ledoux1991probability}. For any fixed realization of the Bernoulli vector $\bm{\delta}$, we have $|\delta_i| \leq 1$. The principle states that for such contractions, the expected supremum does not increase:
\begin{equation}\mathbb{E}_{\bm{\epsilon}}\left[\sup_{f\in\mathcal{F}} \sum_{i=1}^n \epsilon_i (\delta_i f_i)\right] \le \mathbb{E}_{\bm{\epsilon}}\left[\sup_{f\in\mathcal{F}} \sum_{i=1}^n \epsilon_i f_i\right].\end{equation}
Since this inequality holds for any specific outcome of $\bm{\delta}$, it also holds when we take the expectation over $\bm{\delta}$. This gives us:
\begin{equation}
\label{eq:TRC_RC}
E_{TRC} \le \mathbb{E}_{\bm{\epsilon}}\left[\sup_{f\in\mathcal{F}} \sum_{i=1}^n \epsilon_i f_i\right]. \end{equation}
Let $\bm{g}$ be a vector of i.i.d. standard Gaussian variables, and let $\bm{\epsilon}$ be an independent Rademacher vector. The distribution of $\bm{g}$ is identical to the distribution of the vector $(|g_1|\epsilon_1, \dots, |g_n|\epsilon_n)$. Using this property:
\begin{equation}E_{GC} = \mathbb{E}_{\bm{g}}\left[\sup_{f\in\mathcal{F}} \sum_{i=1}^n g_i f_i\right] = \mathbb{E}_{|\bm{g}|, \bm{\epsilon}}\left[\sup_{f\in\mathcal{F}} \sum_{i=1}^n |g_i|\epsilon_i f_i\right].\end{equation}
The function $\phi(z_1, \dots, z_n) = \sup_f \sum_i z_i f_i$ is a convex function of its arguments $z_i$. We can therefore apply Jensen's inequality to the expectation over the magnitudes $|\bm{g}|$:
\begin{equation}E_{GC} \ge \mathbb{E}_{\bm{\epsilon}}\left[\sup_{f\in\mathcal{F}} \sum_{i=1}^n \mathbb{E}[|g_i|] \epsilon_i f_i\right].\end{equation}
The expected absolute value of a standard Gaussian variable is a constant: $\mathbb{E}[|g_i|] = \sqrt{2/\pi}$. Substituting this in gives:
\begin{equation}E_{GC} \ge \mathbb{E}_{\bm{\epsilon}}\left[\sup_{f\in\mathcal{F}} \sum_{i=1}^n \sqrt{\frac{2}{\pi}} \epsilon_i f_i\right] = \sqrt{\frac{2}{\pi}} \cdot \mathbb{E}_{\bm{\epsilon}}\left[\sup_{f\in\mathcal{F}} \sum_{i=1}^n \epsilon_i f_i\right].\end{equation}
Rearranging this inequality and combining it with the result from Eq. (\ref{eq:TRC_RC}), we get:
\begin{equation}E_{TRC} \le \mathbb{E}_{\bm{\epsilon}}\left[\sup_{f\in\mathcal{F}} \sum_{i=1}^n \epsilon_i f_i\right] \le \sqrt{\frac{\pi}{2}} \cdot E_{GC}.\end{equation}

We now have the relationship $E_{TRC} \le \sqrt{\pi/2} \cdot E_{GC}$. We relate these unnormalized expectations back to the full complexity measures:
\begin{align*}
    \mathfrak{R}_{m,n}(\mathcal{F}) &= \left(\frac{1}{m}+\frac{1}{u}\right) \cdot E_{TRC} = \frac{n}{mu} \cdot E_{TRC}, \\
    \mathcal{G}_V(\mathcal{F}) &= \frac{1}{n} \cdot E_{GC} \implies E_{GC} = n \cdot \mathcal{G}_V(\mathcal{F}).
\end{align*}
Substituting these into our combined inequality:
\begin{equation}\frac{mu}{n} \mathfrak{R}_{m,n}(\mathcal{F}) \le \sqrt{\frac{\pi}{2}} \cdot \left( n \cdot \mathcal{G}_V(\mathcal{F}) \right).\end{equation}
Finally, solving for $\mathfrak{R}_{m,n}(\mathcal{F})$ yields the desired result:
\begin{equation}\mathfrak{R}_{m,n}(\mathcal{F}) \le \frac{n^2}{mu} \sqrt{\frac{\pi}{2}} \cdot \mathcal{G}_V(\mathcal{F}).\end{equation}
This completes the proof.
\end{proof}

\section{Proof of Lemma \ref{lem:gci} (FTGC Invariance)}
\label{app:proof_of_gci}

\gci*

\begin{proof}
The FTGC over the entire vertex set $V$ of size $n$ is defined as:
\begin{equation} \mathcal{G}_V(\mathcal{F}) = \mathbb{E}_{\bm{g}} \left[ \sup_{f \in \mathcal{F}} \frac{1}{n} \sum_{i=1}^n g_i (f(\matH^{(0)}))_i \right], \end{equation}
where $g_1, \dots, g_n \sim \mathcal{N}(0,1)$ are i.i.d. standard Gaussian variables. Let $\bm{g} \in \reals^n$ be the vector of these variables. Let $\hat{\bm{y}} = f(\matH^{(0)})$ be an output from the function class $\mathcal{F}$, and let $\gft{\hat{\bm{y}}} = \matU^\top \hat{\bm{y}}$ be its representation in the Fourier domain, which belongs to the class $\gft{\mathcal{F}}$. The proof relies on the rotational invariance of the standard multivariate Gaussian distribution.
\begin{align*}
n \cdot \mathcal{G}_V(\mathcal{F}) &= \mathbb{E}_{\bm{g}} \left[ \sup_{\hat{\bm{y}} \in \mathcal{F}} \langle \bm{g}, \hat{\bm{y}} \rangle \right] \\
&= \mathbb{E}_{\bm{g}} \left[ \sup_{\gft{\hat{\bm{y}}} \in \gft{\mathcal{F}}} \langle \bm{g}, \matU \gft{\hat{\bm{y}}} \rangle \right] \\
&= \mathbb{E}_{\bm{g}} \left[ \sup_{\gft{\hat{\bm{y}}} \in \gft{\mathcal{F}}} \langle \matU^\top \bm{g}, \gft{\hat{\bm{y}}} \rangle \right] \\
&= \mathbb{E}_{\bm{g}'} \left[ \sup_{\gft{\hat{\bm{y}}} \in \gft{\mathcal{F}}} \langle \bm{g}', \gft{\hat{\bm{y}}} \rangle \right] && \text{Let } \bm{g}' = \matU^\top \bm{g}. \text{ Since } \matU \text{ is orthonormal, } \bm{g}' \stackrel{d}{=} \bm{g}. \\
&= n \cdot \mathcal{G}_V(\gft{\mathcal{F}}).
\end{align*}
Therefore, it follows that $\mathcal{G}_V(\mathcal{F}) = \mathcal{G}_V(\gft{\mathcal{F}})$.
\end{proof}

\section{Proof of Lemma \ref{lem:lip_pres} (Lipschitz Preservation of Transformed Activation)}
\label{app:proof_of_lip_pres}

\lippres*

\begin{proof}
We show that for any matrices $\bm{X}, \bm{Y}$ of the same dimensions, the function $\tau$ satisfies $\|\tau(\bm{X}) - \tau(\bm{Y})\|_F \leq \alpha \|\bm{X} - \bm{Y}\|_F$.
\begin{align*}
    \|\tau(\bm{X}) - \tau(\bm{Y})\|_F &= \|\matU^\top \sigma(\matU\bm{X}) - \matU^\top \sigma(\matU\bm{Y})\|_F \\
    &= \|\matU^\top (\sigma(\matU\bm{X}) - \sigma(\matU\bm{Y}))\|_F \\
    &= \|\sigma(\matU\bm{X}) - \sigma(\matU\bm{Y})\|_F && \text{Unitary invariance of } \|\cdot\|_F \\
    &\leq \alpha \|\matU\bm{X} - \matU\bm{Y}\|_F && \sigma \text{ is } \alpha\text{-Lipschitz element-wise} \\
    &= \alpha \|\matU(\bm{X} - \bm{Y})\|_F \\
    &= \alpha \|\bm{X} - \bm{Y}\|_F. && \text{Unitary invariance of } \|\cdot\|_F
\end{align*}
This completes the proof.
\end{proof}

\section{Proof of Theorem \ref{thm:data_dependent_bound}}
\label{app:proof_of_thm3}

\datadependentbound*

\begin{proof}
The proof proceeds by bounding the Frobenius norm of the final layer's features in the Fourier domain, $\|\gft{\bm{h}}^{(L)}\|_2$, and then relating this quantity to the Gaussian complexity.

First, we connect the Gaussian complexity to the norm of the output features. For a function class producing matrices, a standard result from statistical learning theory bounds the complexity by the expected correlation with a random Gaussian matrix $\bm{g} \in \reals^{n}$ \citep{shalev2014understanding}. Using the Cauchy-Schwarz inequality and the fact that $\mathbb{E}[\|\bm{g}\|_2] \le \sqrt{n}$, we get:
\begin{equation} \mathcal{G}_V(\mathcal{F}) \le \frac{1}{n} \mathbb{E}[\|\bm{g}\|_2] \sup_{\bm{\theta}} \|\gft{\bm{h}}^{(L)}\|_2 \le \frac{1}{\sqrt{n}} \sup_{\bm{\theta}} \|\gft{\bm{h}}^{(L)}\|_2. \end{equation}
Our task now is to bound $\sup_{\bm{\theta}} \|\gft{\bm{h}}^{(L)}\|_2$ by unrolling the network's recursion from  Eq. (\ref{eq:fourier_prop_tau}). For any intermediate layer $l \in \{1, \dots, L-1\}$, we can derive a general, data-independent bound. Starting with the propagation rule and using the $\alpha$-Lipschitz property of $\tau$ (Lemma 2), we have:
\begin{equation} \|\gft{\matH}^{(l+1)}\|_F = \|\tau(\text{diag}(\matVp \vectheta^{(l)}) \gft{\matH}^{(l)} \matW^{(l)})\|_F \le \alpha \| \text{diag}(\matVp \vectheta^{(l)}) \gft{\matH}^{(l)} \matW^{(l)} \|_F. \end{equation}
Using the submultiplicative property of matrix norms ($ \|ABC\|_F \le \|A\|_2 \|B\|_F \|C\|_2 $) \citep{horn2012matrix}, this becomes:
\begin{equation} \|\gft{\matH}^{(l+1)}\|_F \le \alpha \|\text{diag}(\matVp \vectheta^{(l)})\|_2 \|\gft{\matH}^{(l)}\|_F \|\matW^{(l)}\|_2. \end{equation}
The spectral norm of the diagonal matrix is its largest absolute entry. We can bound this using the Cauchy-Schwarz inequality: $\|\text{diag}(\matVp \vectheta^{(l)})\|_2 = \max_i |\langle \bm{v}_i, \vectheta^{(l)} \rangle| \le (\max_i \|\bm{v}_i\|_2) \|\vectheta^{(l)}\|_2 = \|\matVp\|_{2,\infty} \|\vectheta^{(l)}\|_2$. This gives the recursive bound for intermediate layers:
\begin{equation} \|\gft{\matH}^{(l+1)}\|_F \le \alpha \|\matW^{(l)}\|_2 \|\vectheta^{(l)}\|_2 \|\matVp\|_{2,\infty} \|\gft{\matH}^{(l)}\|_F. \end{equation}
Unrolling this recursion from $l=L-1$ down to $l=1$ gives:
\begin{equation} \|\gft{\matH}^{(L)}\|_F \le \left( \prod_{l=1}^{L-1} \alpha \|\matW^{(l)}\|_2 \|\vectheta^{(l)}\|_2 \|\matVp\|_{2,\infty} \right) \|\gft{\matH}^{(1)}\|_F. \end{equation}
For the first layer ($l=0$), we derive a tighter, data-dependent bound for $\|\gft{\matH}^{(1)}\|_F$. Following the same initial steps:
\begin{equation} \|\gft{\matH}^{(1)}\|_F \le \alpha \| \text{diag}(\matVp \vectheta^{(0)}) \gft{\matH}^{(0)} \matW^{(0)} \|_F \le \alpha \|\matW^{(0)}\|_2 \| \text{diag}(\matVp \vectheta^{(0)}) \gft{\matH}^{(0)} \|_F. \end{equation}
We now analyze the squared norm of the filtered signal term by expanding it:
\begin{align*}
\|\text{diag}(\matVp \vectheta^{(0)}) \gft{\matH}^{(0)}\|_F^2 &= \sum_{i=1}^n \sum_{j=1}^{d_0} |(\text{diag}(\matVp \vectheta^{(0)}))_{ii} (\gft{\matH}^{(0)})_{ij}|^2 \\
&= \sum_{i=1}^n |(\matVp \vectheta^{(0)})_i|^2 \sum_{j=1}^{d_0} |(\gft{\matH}^{(0)})_{ij}|^2 \\
&= \sum_{i=1}^n |\langle \bm{v}_i, \vectheta^{(0)} \rangle|^2 \|(\gft{\matH}^{(0)})_i\|_2^2 \\
&\le \sum_{i=1}^n (\|\bm{v}_i\|_2^2 \|\vectheta^{(0)}\|_2^2) \mathcal{E}_0(\lambda_i) && \text{(by Cauchy-Schwarz)} \\
&= \|\vectheta^{(0)}\|_2^2 \sum_{i=1}^n \|\bm{v}_i\|_2^2 \mathcal{E}_0(\lambda_i).
\end{align*}
Taking the square root provides the bound on the norm of the feature matrix after the first layer:
\begin{equation} \|\gft{\matH}^{(1)}\|_F \le \alpha \|\matW^{(0)}\|_2 \|\vectheta^{(0)}\|_2 \left( \sum_{i=1}^n \|\bm{v}_i\|_2^2 \mathcal{E}_0(\lambda_i) \right)^{1/2}. \end{equation}
Finally, substituting this data-dependent bound into the recursive inequality for $\|\gft{\matH}^{(L)}\|_F$, applying the parameter constraints ($C_{W,l}, C_{\theta,l}$) and plugging the result into the initial inequality for $\mathcal{G}_V(\mathcal{F})$ completes the proof.
\end{proof}

\section{Proof of Theorem \ref{thm:tighter_bound_deep}}
\label{app:proof_of_thm_linear}

\tighterbounddeep*

\begin{proof}
We work in the graph Fourier domain throughout. Recall that the (linear) $L$-layer model is
\[
\widehat{\bm H}^{(l+1)} \;=\; \bm{D}_l\,\widehat{\bm H}^{(l)}\,\bm W^{(l)},
\qquad
\bm{D}_l \;=\; \mathrm{diag}\!\big(\matV_P\,\vectheta^{(l)}\big),
\]
with $\widehat{\bm H}^{(0)}\in\mathbb{R}^{n\times d_0}$, $\bm W^{(l)}\in\mathbb{R}^{d_l\times d_{l+1}}$,
and $d_L=1$. Let $\bm{v}_i^\top$ denote the $i$-th row of $\matV_P$ and
$\mathcal{E}_0(\lambda_i) := \|(\widehat{\bm H}^{(0)})_i\|_2^2$ the input spectral energy at frequency $\lambda_i$.

Unrolling gives $\widehat{\bm h}^{(L)}=\big(\prod_{l=0}^{L-1}\bm{D}_l\big)\widehat{\bm H}^{(0)}\big(\prod_{l=0}^{L-1}\bm W^{(l)}\big)$.
By spectral-norm duality (applied layer-by-layer),
\begin{equation}
\sup_{\{\bm W^{(l)}:\,\|\bm W^{(l)}\|_2\le C_{W,l}\}}
\!\!\langle\bm g,\widehat{\bm h}^{(L)}\rangle
\;=\;
\Big(\prod_{l=0}^{L-1}C_{W,l}\Big)\,
\big\|\widehat{\bm H}^{(0)\!\top}\big(\textstyle\prod_{l=0}^{L-1}\bm{D}_l\big)\bm g\big\|_2 .
\label{eq:elimW}
\end{equation}
Hence, by the FTGC definition,
\begin{equation}
\mathcal{G}_V(\mathcal{F})
\;\le\;
\frac{1}{n}\Big(\prod_{l=0}^{L-1}C_{W,l}\Big)\,
\mathbb{E}_{\bm g}\Big[
\sup_{\{\bm\theta^{(l)}:\,\|\bm\theta^{(l)}\|_2\le C_{\theta,l}\}}
\big\|\widehat{\bm H}^{(0)\!\top}\big(\textstyle\prod_{l=0}^{L-1}\bm{D}_l\big)\bm g\big\|_2
\Big].
\label{eq:afterW}
\end{equation}

Next, we write the product of diagonals as a single diagonal via row-wise Kronecker lifting.
For each frequency $i$, we have
\[
\big(\textstyle\prod_{l=0}^{L-1} \bm{D}_l\big)_{ii}
\;=\;\prod_{l=0}^{L-1} \langle \bm{v}_i,\vectheta^{(l)}\rangle
\;=\; \big\langle \bm{v}_i^{\otimes L},\,\vectheta^{(L-1)}\otimes\cdots\otimes\vectheta^{(0)}\big\rangle.
\]
Let
\[
\Theta \;:=\; \vectheta^{(L-1)}\otimes\cdots\otimes\vectheta^{(0)}\in\mathbb{R}^{(K+1)^L},
\qquad
\matV_{\!\otimes}
\;:=\; \underbrace{\matV_P \odot_r \matV_P \odot_r \cdots \odot_r \matV_P}_{L\text{ times}}
\in\mathbb{R}^{n\times (K+1)^L},
\]
where $\odot_r$ denotes the row-wise Khatri–Rao/Kronecker \citep{khatrikolda2009} product so that the $i$-th row of $\matV_{\!\otimes}$ equals $\bm{v}_i^{\otimes L}$ and $\|\bm{v}_i^{\otimes L}\|_2 =\|\bm{v}_i\|^L_2$. Then
\[
\prod_{l=0}^{L-1} \bm{D}_l \;=\; \mathrm{diag}(\matV_{\!\otimes}\Theta).
\]
Thus, for each fixed $\bm g$,
\[
\sup_{\{\vectheta^{(l)}\}}
\big\|\widehat{\bm H}^{(0)\top}\big(\prod_{l=0}^{L-1}\bm{D}_l\big)\bm g\big\|_2
\;=\;
\sup_{\{\vectheta^{(l)}\}}
\big\|\widehat{\bm H}^{(0)\top}\mathrm{diag}(\bm g)\,\matV_{\!\otimes}\Theta\big\|_2
\;\le\;
\Big(\prod_{l=0}^{L-1}C_{\theta,l}\Big)\;
\big\|\widehat{\bm H}^{(0)\top}\mathrm{diag}(\bm g)\,\matV_{\!\otimes}\big\|_2,
\]
since $\|\Theta\|_2=\prod_{l=0}^{L-1}\|\vectheta^{(l)}\|_2\le\prod_l C_{\theta,l}$.

Taking expectation in Eq. (\ref{eq:afterW}) gives
\begin{equation}
\mathcal{G}_V(\mathcal{F})
\;\le\;
\frac{1}{n}\Big(\prod_{l=0}^{L-1} C_{W,l}C_{\theta,l}\Big)\,
\mathbb{E}_{\bm g}\Big[
\big\|\widehat{\bm H}^{(0)\top}\mathrm{diag}(\bm g)\,\matV_{\!\otimes}\big\|_2
\Big].
\label{eq:preFrob}
\end{equation}

By using $\|\bm{X}\|_2\le \|\bm{X}\|_F$ and Jensen/Cauchy--Schwarz:
\[
\mathbb{E}_{\bm g}\big[\|\widehat{\bm H}^{(0)\top}\mathrm{diag}(\bm g)\,\matV_{\!\otimes}\|_2\big]
\;\le\;
\sqrt{\,\mathbb{E}_{\bm g}\big[\|\widehat{\bm H}^{(0)\top}\mathrm{diag}(\bm g)\,\matV_{\!\otimes}\|_F^2\big]\,}.
\]
Let $\bm{Q}:=\widehat{\bm H}^{(0)}\widehat{\bm H}^{(0)\top}$.
Then
\[
\|\widehat{\bm H}^{(0)\top}\mathrm{diag}(\bm g)\,\matV_{\!\otimes}\|_F^2
=
\mathrm{Tr}\!\Big(\matV_{\!\otimes}^{\!\top}\mathrm{diag}(\bm g)\,\bm{Q}\,\mathrm{diag}(\bm g)\,\matV_{\!\otimes}\Big).
\]
Since $\mathbb{E}[g_i g_j]=\delta_{ij}$,
\[
\mathbb{E}\big[\mathrm{diag}(\bm g)\,\bm{Q}\,\mathrm{diag}(\bm g)\big]
=
\mathrm{diag}(\bm{Q}).
\]
Therefore
\begin{equation}
\mathbb{E}_{\bm g}\|\widehat{\bm H}^{(0)\top}\mathrm{diag}(\bm g)\,\matV_{\!\otimes}\|_F^2
=
\mathrm{Tr}\!\Big(\matV_{\!\otimes}^{\!\top}\mathrm{diag}(\bm{Q})\,\matV_{\!\otimes}\Big)
=
\sum_{i=1}^n \bm{Q}_{ii}\;\big\|\big(\matV_{\!\otimes}\big)_{i,:}\big\|_2^2
\;=\;
\sum_{i=1}^n \|(\widehat{\bm H}^{(0)})_i\|_2^2 \;\|\bm{v}_i^{\otimes L}\|_2^2.
\end{equation}
We have $\|(\widehat{\bm H}^{(0)})_i\|_2^2=\mathcal{E}_0(\lambda_i)$ and
$\|\bm{v}_i^{\otimes L}\|_2^2=\|\bm{v}_i\|_2^{2L}$, so
\[
\mathbb{E}_{\bm g}\|\widehat{\bm H}^{(0)\top}\mathrm{diag}(\bm g)\,\matV_{\!\otimes}\|_F^2
\;=\;
\sum_{i=1}^n \|\bm{v}_i\|_2^{2L}\,\mathcal{E}_0(\lambda_i).
\]
Combining with Eq. (\ref{eq:preFrob}) yields
\[
\mathcal{G}_V(\mathcal{F})
\;\le\;
\frac{1}{n}\Big(\prod_{l=0}^{L-1} C_{W,l}C_{\theta,l}\Big)\,
\left(\sum_{i=1}^n \|\bm{v}_i\|_2^{2L}\,\mathcal{E}_0(\lambda_i)\right)^{1/2}.
\]
\end{proof}

\section{Proof of Theorem \ref{thm:jacobian_bound}}
\label{app:proof_of_thm4}

\jacobianbound*

\begin{proof}
The proof proceeds by analyzing the Jacobian in the graph Fourier domain, where the graph convolution operation simplifies to an element-wise product, and then combining bounds for each layer using the chain rule.

Our first step is to move the analysis to the graph Fourier domain. The Jacobian in the Fourier domain is defined as $\gft{\mathcal{J}} = \frac{\partial \text{vec}(\gft{\matH}^{(L)})}{\partial \text{vec}(\gft{\matH}^{(0)})}$. The vectorized features in the two domains are related by the linear transformation $\text{vec}(\gft{\matH}^{(l)}) = (\bm{I} \otimes \matU^\top) \text{vec}(\matH^{(l)})$, where $\matU$ is the orthogonal matrix of normalized Adjacency eigenvectors. By the multivariate chain rule, the Jacobians are related by a similarity transformation:
\begin{equation} \gft{\mathcal{J}} = (\bm{I} \otimes \matU^\top) \mathcal{J} (\bm{I} \otimes \matU). \end{equation}
Since $\matU$ is an orthogonal (and thus unitary) matrix, the Kronecker product $(\bm{I} \otimes \matU)$ is also unitary. A fundamental property of the spectral norm is its unitary invariance, meaning $\|\bm{V}^H \bm{A} \bm{V}\|_2 = \|\bm{A}\|_2$. for any unitary matrix $\bm{V}$. Therefore, the spectral norms of the Jacobians in the two domains are identical:
\begin{equation} \|\mathcal{J}\|_2 = \|\gft{\mathcal{J}}\|_2 \end{equation}
This equivalence allows us to derive the bound in the Fourier domain without loss of generality.

By the chain rule, the full Jacobian can be expressed as the product of the Jacobians of individual layers:
\begin{equation} \gft{\mathcal{J}} = \prod_{l=0}^{L-1} \gft{\mathcal{J}}^{(l)} \quad \text{where} \quad \gft{\mathcal{J}}^{(l)} = \frac{\partial \text{vec}(\gft{\matH}^{(l+1)})}{\partial \text{vec}(\gft{\matH}^{(l)})}. \end{equation}
Using the submultiplicative property of the spectral norm, we can bound the total norm by the product of the individual layer norms: $\|\gft{\mathcal{J}}\|_2 \le \prod_{l=0}^{L - 1} \|\gft{\mathcal{J}}^{(l)}\|_2$.

We now bound the norm $\|\gft{\mathcal{J}}^{(l)}\|_2$ for a single layer $l$. Let $\gft{S}^{(l)} = \text{diag}(\matVp \vectheta^{(l)}) \gft{\matH}^{(l)} \matW^{(l)}$ be the pre-activation matrix in the Fourier domain. The propagation rule is $\gft{\matH}^{(l+1)} = \tau(\gft{S}^{(l)}) = \matU^\top \sigma(\matU \gft{S}^{(l)})$. The layer-wise Jacobian is given by:
\begin{equation} \gft{\mathcal{J}}^{(l)} = \underbrace{\frac{\partial \text{vec}(\tau(\gft{S}^{(l)}))}{\partial \text{vec}(\gft{S}^{(l)})}}_{J_\tau} \underbrace{\frac{\partial \text{vec}(\gft{S}^{(l)})}{\partial \text{vec}(\gft{\matH}^{(l)})}}_{J_S}. \end{equation}
The second term, $J_S$, is the Jacobian of a linear transformation, which evaluates to the Kronecker product $J_S = (\matW^{(l)T} \otimes \text{diag}(\matVp \vectheta^{(l)}))$, since $\text{vec}(\gft{S}^{(l)}) = \left((\matW^{(l)})^\top \otimes \text{diag}(\matVp \vectheta^{(l)})\right) \text{vec}(\gft{\matH}^{(l)})$. The first term, $J_\tau$, is the Jacobian of the transformed activation, which is $J_\tau = (\bm{I} \otimes \matU^\top) \text{diag}(\text{vec}(\sigma'(\matU \gft{S}^{(l)}))) (\bm{I} \otimes \matU)$.
We can now bound the norm of the product:
\begin{align*}
    \|\gft{\mathcal{J}}^{(l)}\|_2 &\le \|J_\tau\|_2 \cdot \|J_S\|_2 \\
    &= \|\text{diag}(\text{vec}(\sigma'(\matU \gft{S}^{(l)})))\|_2 \cdot \|\matW^{(l)T} \otimes \text{diag}(\matVp \vectheta^{(l)})\|_2 && \text{(Unitary invariance of } J_\tau\text{)} \\
    &= \left( \max_{i,j}|\sigma'((\matU \gft{S}^{(l)})_{ij})| \right) \cdot \|\matW^{(l)}\|_2 \cdot \|\text{diag}(\matVp \vectheta^{(l)})\|_2. && \text{(Norm of diag and Kronecker)}
\end{align*}
Since $\sigma$ is $\alpha$-Lipschitz, its derivative is bounded by $\alpha$. The norm of the diagonal filter matrix is bounded by its largest entry, which via Cauchy-Schwarz is $|\langle \bm{v}_i, \vectheta^{(l)} \rangle| \le \|\bm{v}_i\|_2 \|\vectheta^{(l)}\|_2$. This gives $\|\text{diag}(\matVp \vectheta^{(l)})\|_2 \le \|\vectheta^{(l)}\|_2 \|\matVp\|_{2,\infty}$. Substituting these bounds yields:
\begin{equation} \|\gft{\mathcal{J}}^{(l)}\|_2 \le \alpha \|\matW^{(l)}\|_2 \|\vectheta^{(l)}\|_2 \|\matVp\|_{2,\infty}. \end{equation}

Finally, we substitute the per-layer bound into the product from Step 2:
\begin{equation} \|\mathcal{J}\|_2 = \|\gft{\mathcal{J}}\|_2 \le \prod_{l=0}^{L-1} \|\gft{\mathcal{J}}^{(l)}\|_2 \le  \prod_{l=0}^{L-1} \left( \alpha \|\matW^{(l)}\|_2 \|\vectheta^{(l)}\|_2 \|\matVp\|_{2,\infty} \right). \end{equation}
Applying the parameter constraints $\|\matW^{(l)}\|_2 \le C_{W,l}$ and $\|\vectheta^{(l)}\|_2 \le C_{\theta,l}$ completes the proof.
\end{proof}

\section{Empirical Tightness of the Jacobian Bound}
\label{app:jacobian_tightness}

To empirically evaluate the tightness of the stability bound derived in Theorem~\ref{thm:jacobian_bound}, we computed both our theoretical upper bound and the true spectral norm of the network Jacobian. 

We evaluated a 2-layer nonlinear model utilizing the Monomial basis across varying polynomial orders ($K \in [1, 10]$) on both the Cora and Chameleon datasets. The true spectral norm of the network Jacobian was computed exactly using power iteration via automatic differentiation (to avoid materializing the full matrix in memory). This empirical true norm was then compared directly against our theoretical bound given by $\prod_{l=0}^{L-1} \left( \alpha C_{W,l} C_{\theta,l} \|\matVp\|_{2,\infty} \right)$.

As illustrated in Figure~\ref{fig:jacobian_tightness}, our theoretical bound tracks the true Jacobian norm reliably. On the Cora dataset, the bound is exceptionally tight, remaining within a small factor (1.0x to 2.0x) of the true norm, and it correctly identifies the location of the instability spike at $K=7$. On the Chameleon dataset, it successfully bounds the true norm within a factor of 1.3x to 4.0x. While the bound slightly overestimates the magnitude of instability at certain orders on Chameleon, it provides a tight measure of the model's worst-case sensitivity.

\begin{figure}[htbp]
    \centering
    \begin{subfigure}{0.48\textwidth}
        \centering
        \includegraphics[width=\linewidth]{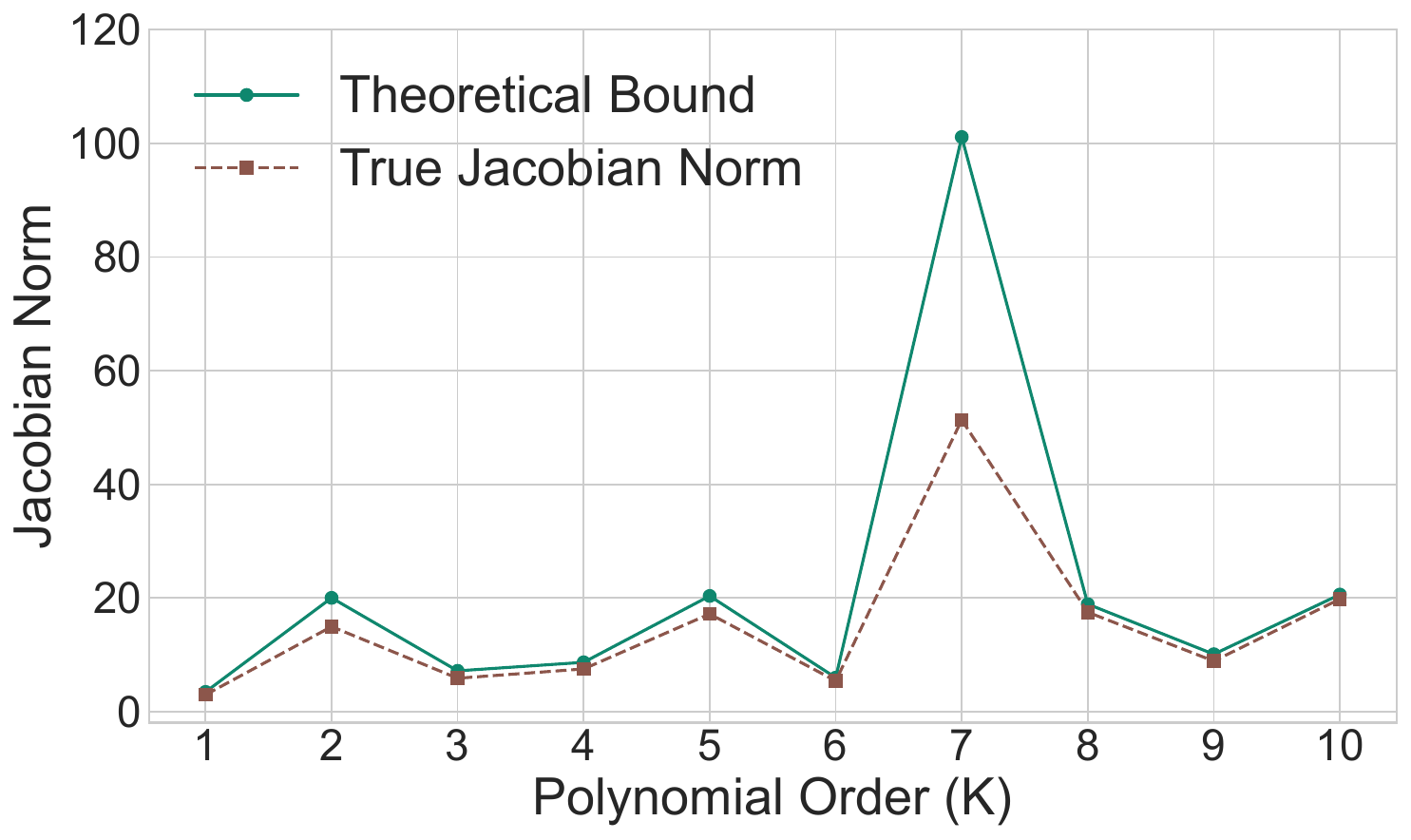}
        \caption{Cora}
    \end{subfigure}\hfill
    \begin{subfigure}{0.48\textwidth}
        \centering
        \includegraphics[width=\linewidth]{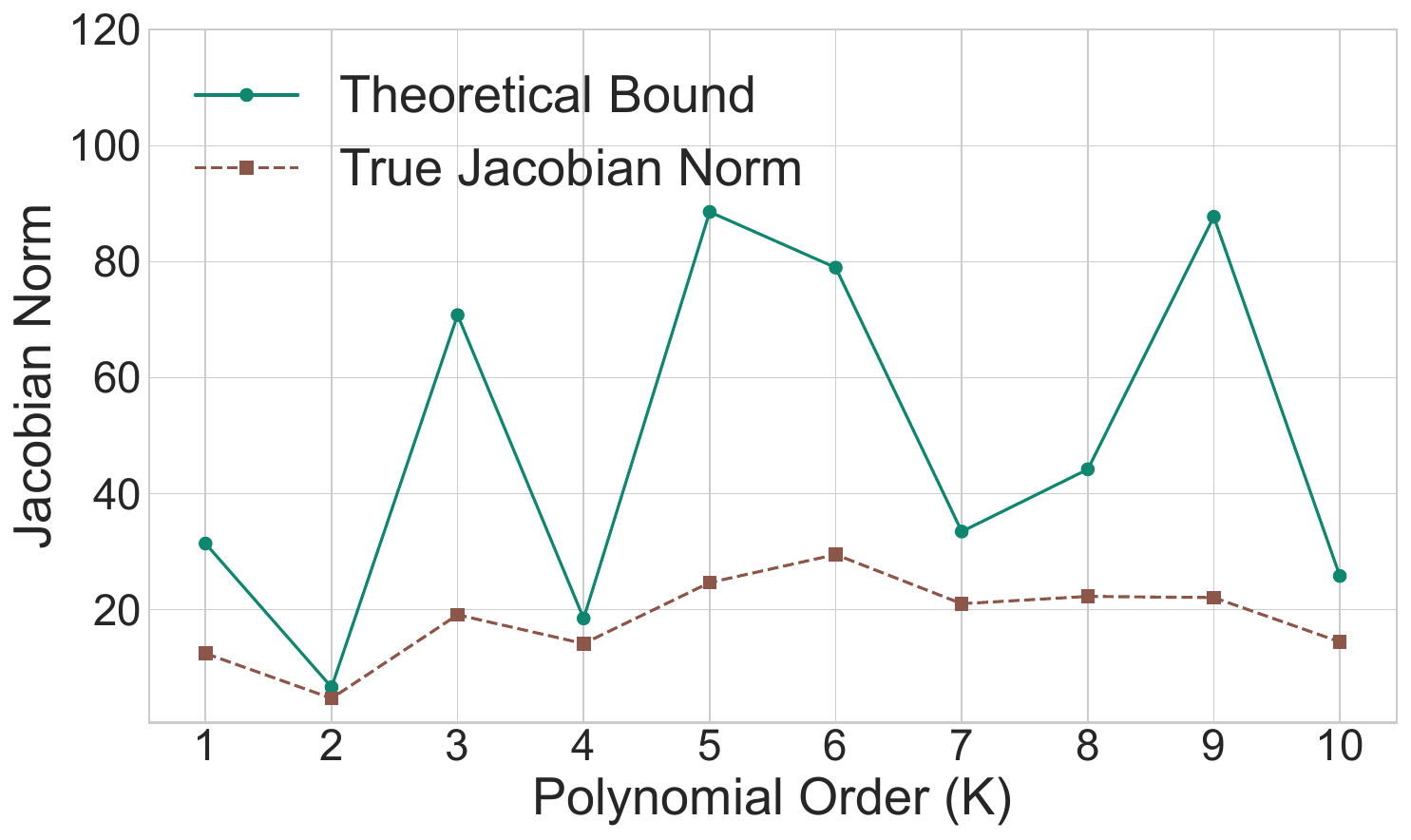}
        \caption{Chameleon}
    \end{subfigure}
    \caption{Jacobian bound tightness.}
    \label{fig:jacobian_tightness}
\end{figure}

\section{Correlation between the Theoretical Bound and the Generalization Gap}
\label{app:boundvsgap}

To compare theory and measurements under controlled conditions, we keep the architecture minimal and the training setup fixed across all plots. Unless noted, we use the spectral layer in Eq. (\ref{eq:gnn_layer}) \emph{without} activations for the linear setting (Section~\ref{sec:main}) and the standard model \emph{with} activations for the nonlinear setting. The hyperparameters are the following: learning rate $0.01$, weight decay $10^{-5}$, maximum $500$ epochs with early stopping after $100$ epochs without validation improvement. Results are averaged over $30$ random splits (compared to 10 random splits in Section \ref{sec:empirical_validation}). The generalization gap is defined as
$
\text{Gap} = \text{test\_loss} - \text{train\_loss}.
$
No extra regularizers (beyond weight decay) are used in these plots. Error bars in the accuracy panels are standard deviations across splits.

For the linear experiments (Cora and Chameleon, $L\in\{1,2\}$), we compare directly against the tighter linear FTGC bound (Theorem~\ref{thm:tighter_bound_deep}), which isolates depth, basis amplification, and the input spectrum (see Figures~\ref{fig:cora_linear} and \ref{fig:chameleon_linear}). The bound is evaluated via
\[
\mathcal{G}_V(\mathcal{F}) \;\le\; \frac{1}{n}\!\left(\prod_{l=0}^{L-1} C_{W,l}C_{\theta,l}\right)
\!\left(\sum_{i=1}^n \|\bm v_i\|_2^{\,2L}\,\mathcal{E}_0(\lambda_i)\right)^{\!1/2}.
\]
For the nonlinear experiments ($L=2$ only), we use the nonlinear bound from Theorem~\ref{thm:data_dependent_bound} (see Figures~\ref{fig:cora_nonlinear} and \ref{fig:chameleon_nonlinear}). We only show $L=2$ for the nonlinear case because with $L=1$ the neural network is effectively linear and does not add anything new to this comparison. Across all runs, the models fit the labeled nodes well. On Chameleon, the minimum training accuracy is $76\%$, and on Cora, it is $94\%$. Since these are much larger than their corresponding test accuracies, the positive gaps are not due to underfitting.

\paragraph{Sensitivity Analysis.} 
To verify that our Fourier Transductive Gaussian Complexity (FTGC) bound truly captures the generalization gap via spectral properties rather than just parameter norms, we performed a sensitivity analysis by decomposing the bound into the Weight Term ($\prod \alpha C_{W,l}C_{\theta,l}$) and the Spectral Interaction Term ($\|\matVp\|_{2,\infty}\sqrt{\sum\|\bm v_i\|_2^2\mathcal{E}_0(\lambda_i)}$). For a 2-layer nonlinear model on Cora, the Pearson correlation drops drastically from $0.86$ (Full Bound, Figure~\ref{fig:cora_nonlinear}a) to $-0.19$ when using only the Weight Term. Similarly, on Chameleon, the correlation drops from $0.36$ (Figure~\ref{fig:chameleon_nonlinear}a) to $-0.53$. This confirms that the data-dependent spectral term is the crucial driver for predicting the generalization gap, validating our Fourier-domain analysis.

\paragraph{Large-Scale Validation on ogbn-arxiv.} 
To verify that our theoretical framework scales to larger networks, we extended our empirical validation to the ogbn-arxiv dataset (approximately 170,000 nodes). We maintained the same highly sparse label setting used for the smaller datasets (only 10 training nodes per class) to rigorously stress-test our bound.  Using the practical $L=2$ nonlinear model, we computed the FTGC bound across different polynomial bases. As shown in Figure~\ref{fig:ogbn_arxiv_nonlinear}, we observed a strong positive correlation, achieving a Pearson coefficient of $0.85$ and a Spearman rank coefficient of $0.71$. This confirms that our bound robustly tracks generalization behavior even on massive graphs.

\paragraph{Main observations.} The results shown in Figures~\ref{fig:cora_linear}--\ref{fig:ogbn_arxiv_nonlinear} suggest the following main observations:
\begin{enumerate}[leftmargin=0.8cm,itemsep=3pt]
\item In the bound–gap panels, we observe a clear positive correlation between our bound and the empirical gap on all evaluated datasets. We report the overall Pearson $r$ and Spearman $\rho$ (with 95\% CIs by Fisher transform) inside panels (a) and (c) of each figure.
\item The correlation is visibly stronger for $L=2$ than for $L=1$ (Figures \ref{fig:cora_linear} and \ref{fig:chameleon_linear}). This matches Theorem~\ref{thm:tighter_bound_deep}, as the basis term scales as $\|\bm v_i\|_2^{2L}$ and magnifies differences across bases and spectra with depth.
\item Points cluster by polynomial basis in the bound–gap plane, reflecting their amplification profiles $x\mapsto\sum_{k=0}^K P_k(x)^2$ (Figure~\ref{fig:basis_stability} of the main paper). Furthermore, stable bases like Chebyshev exhibit tighter correlations within their own group. This predictable behavior occurs because the spectral amplification profile of Chebyshev is relatively uniform (nearly constant) across the spectrum, preventing extreme fluctuations in the theoretical bound.

\item Chebyshev’s relatively uniform profile does not downweight any part of the spectrum. Because it amplifies all frequencies nearly equally (including the middle spectrum), the term $\sum_i \|\bm v_i\|_2^{2L}\mathcal{E}_0(\lambda_i)$ becomes strictly larger when $\mathcal{E}_0(\lambda_i)$ spreads broadly. This perfectly explains why it tends to yield both higher theoretical bounds and higher empirical generalization gaps (as seen in the top-right clusters of the plots).

\item Cora is homophilous, meaning low-frequency content is more predictive. Bases with stronger low-pass behavior (more U-shaped amplification in Figure~\ref{fig:basis_stability} of the main paper) tend to achieve higher test accuracy.
\end{enumerate}

\begin{center}
  \begin{minipage}{0.93\textwidth}
    \centering
    \captionsetup{hypcap=false}
    \includegraphics[width=\linewidth]{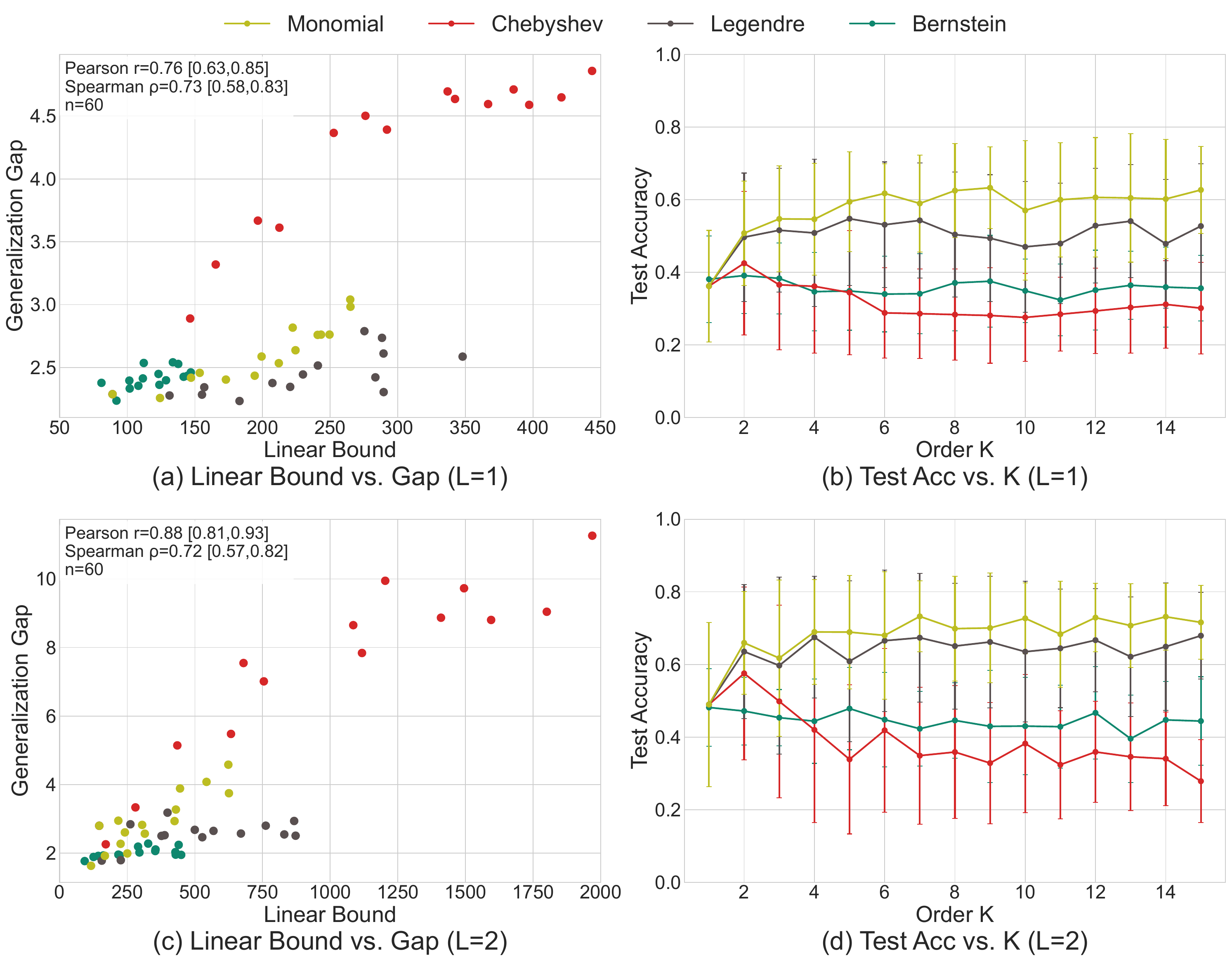}
    \captionof{figure}{Cora linear models.}
    \label{fig:cora_linear}
  \end{minipage}

  \vspace{0.5em}

  \begin{minipage}{0.93\textwidth}
    \centering
    \captionsetup{hypcap=false}
    \includegraphics[width=\linewidth]{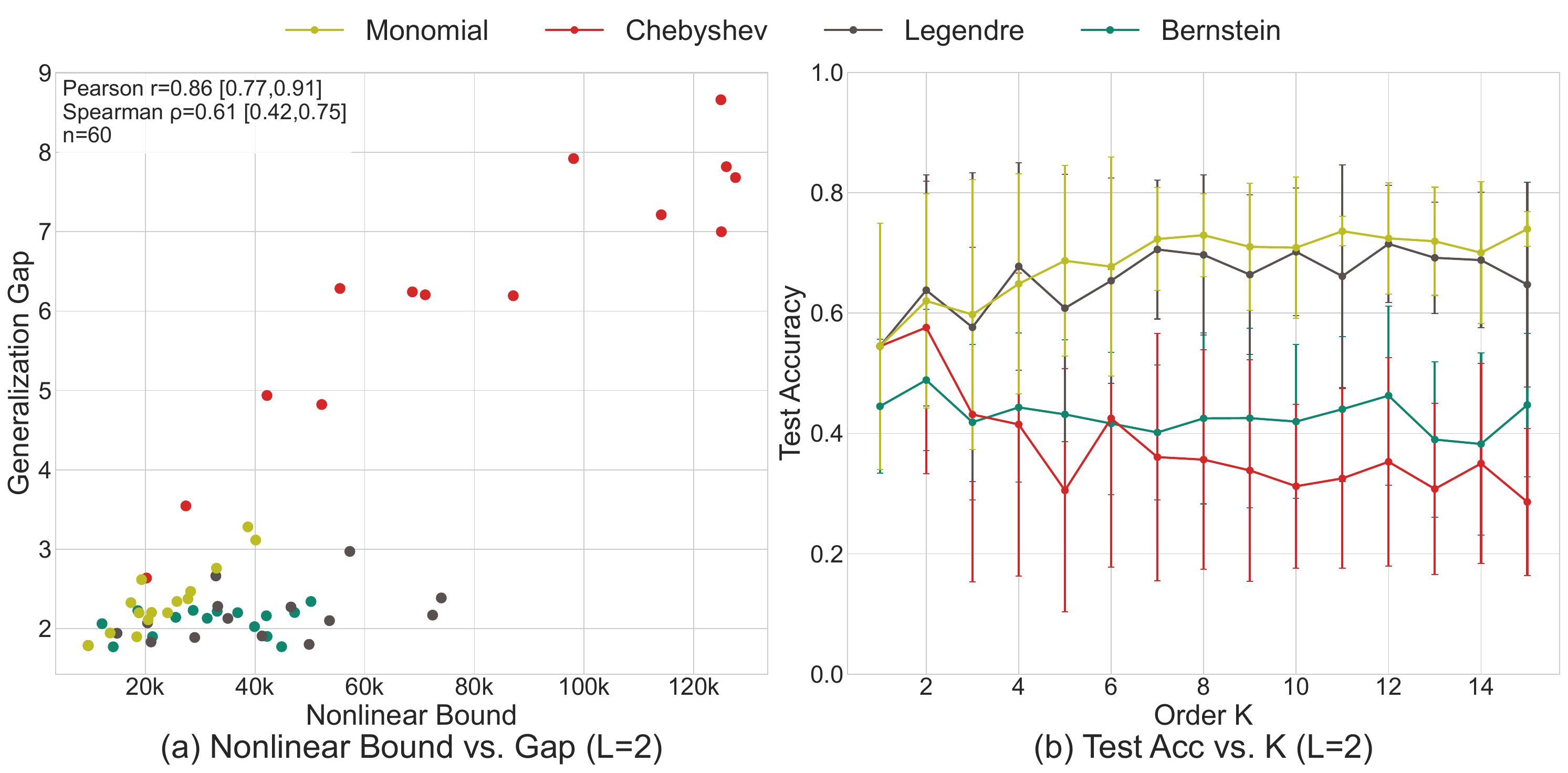}
    \captionof{figure}{Cora nonlinear model.}
    \label{fig:cora_nonlinear}
  \end{minipage}
\end{center}

\clearpage
\begin{center}
  \begin{minipage}{0.93\textwidth}
    \centering
    \captionsetup{hypcap=false}
    \includegraphics[width=\linewidth]{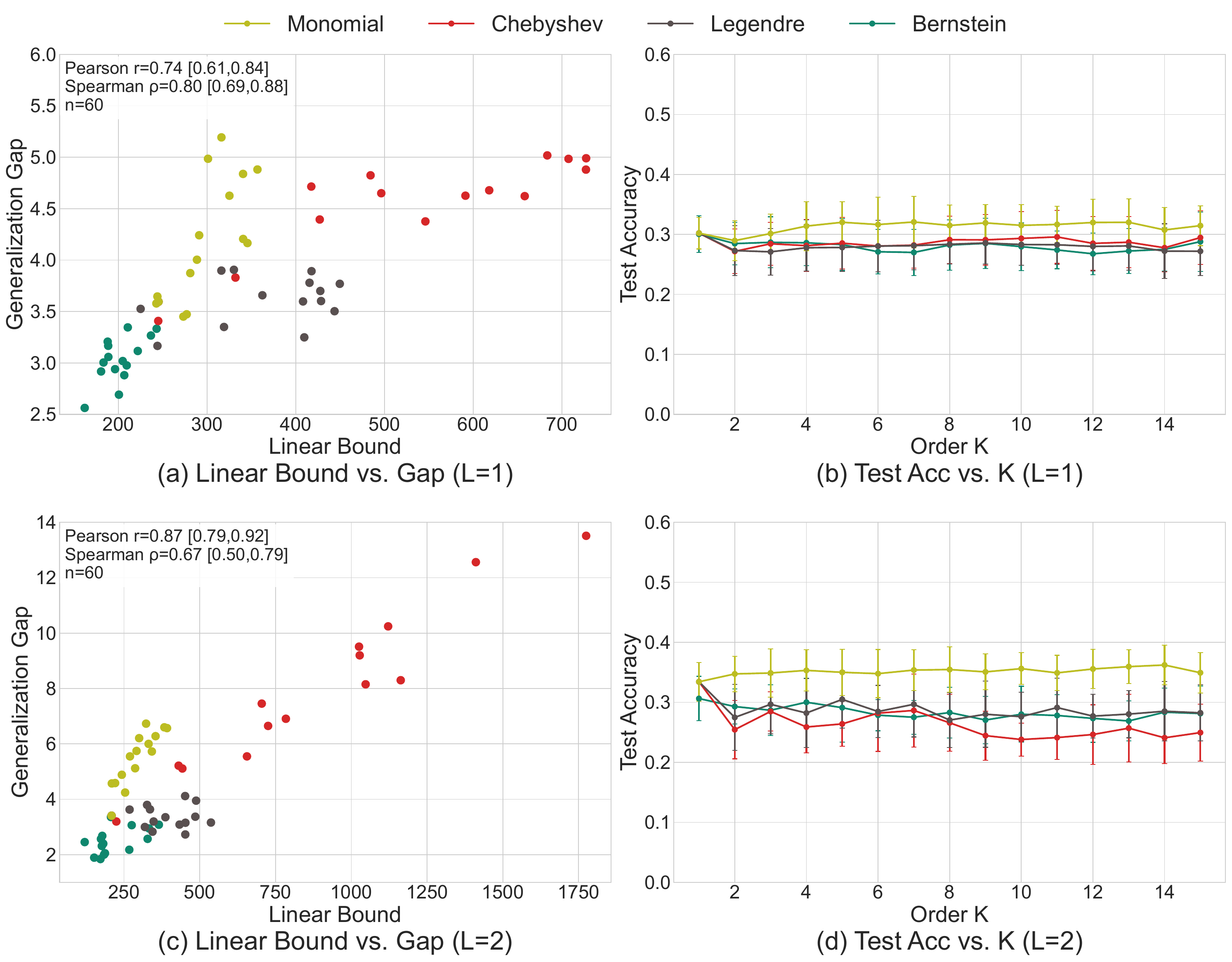}
    \captionof{figure}{Chameleon linear models.}
    \label{fig:chameleon_linear}
  \end{minipage}

  \vspace{0.5em}

  \begin{minipage}{0.93\textwidth}
    \centering
    \captionsetup{hypcap=false}
    \includegraphics[width=\linewidth]{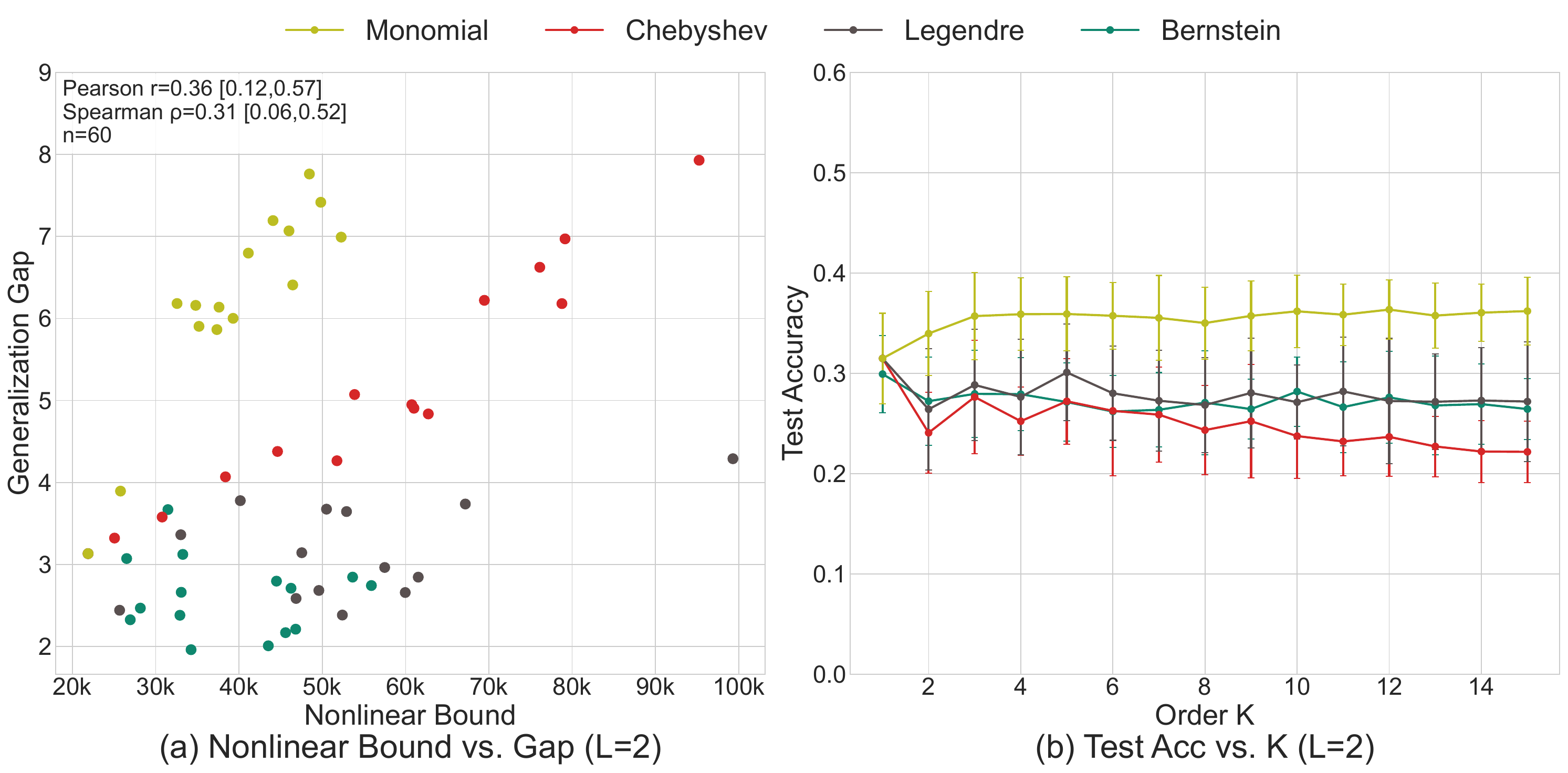}
    \captionof{figure}{Chameleon nonlinear model.}
    \label{fig:chameleon_nonlinear}
  \end{minipage}
\end{center}

\begin{center}
  \begin{minipage}{0.93\textwidth}
    \centering
    \captionsetup{hypcap=false}
    \includegraphics[width=\linewidth]{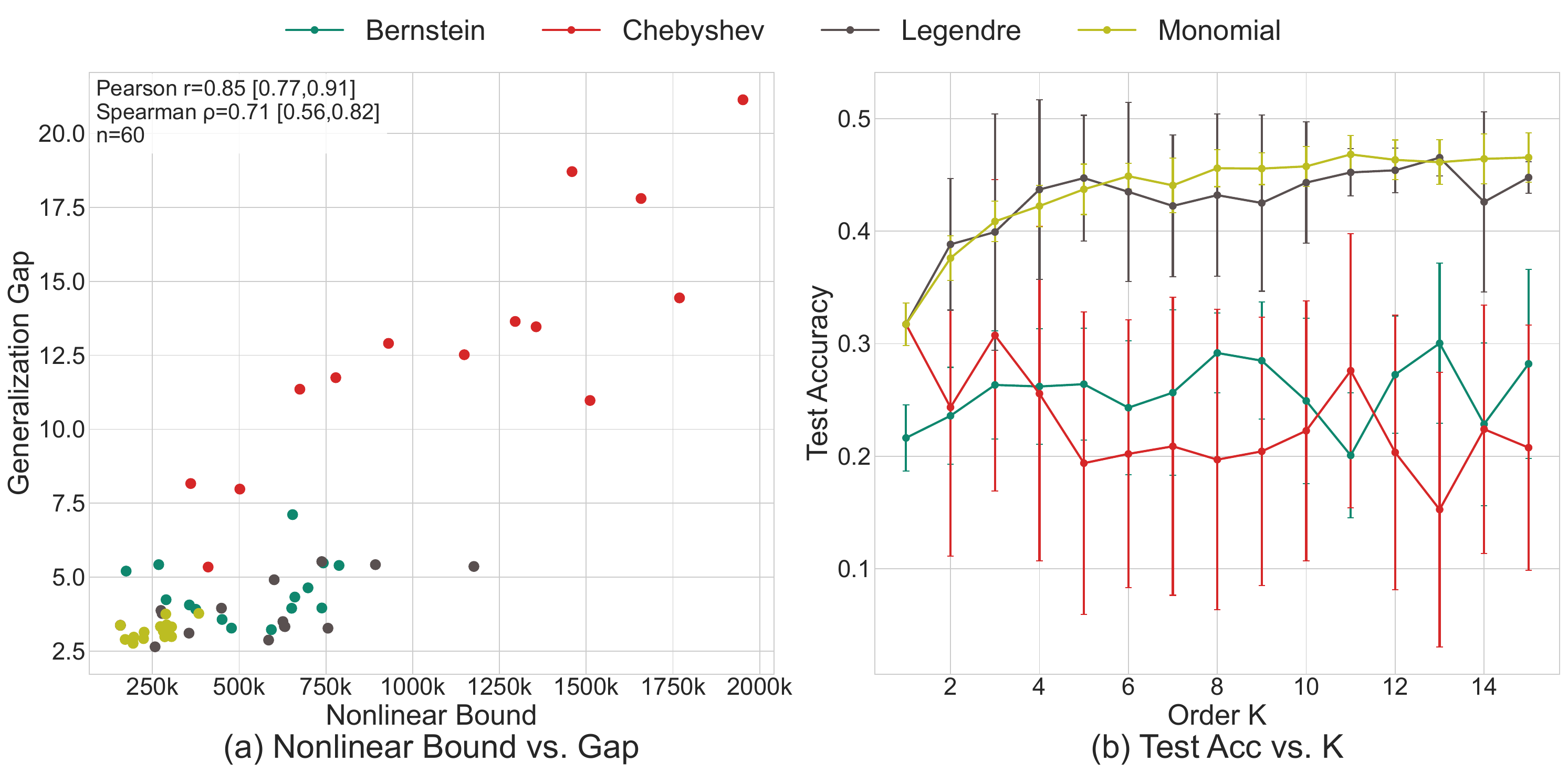} 
    \captionof{figure}{Ogbn-arxiv nonlinear model.}
    \label{fig:ogbn_arxiv_nonlinear}
  \end{minipage}
\end{center}

\section{Comparison with Prior Transductive Bounds}
\label{app:bound_comparison}

To quantitatively evaluate the tightness of our proposed Fourier-Transductive Gaussian Complexity (FTGC) bound, we compare it against the generalization bound derived specifically for GNNs by \cite{Esser2021LearningTheory}. For a fair comparison, we employ a standard GCN architecture, as the theoretical analysis by \cite{Esser2021LearningTheory} is limited to GCNs. We calculate both bounds across varying network depths ($L \in [1, 7]$) on the Cora and Chameleon datasets. 

As shown in \Cref{fig:bound_comparison}, as network depth increases, the spatial bound from \cite{Esser2021LearningTheory} suffers from exponential amplification due to its dependency on the spatial graph shift operator ($\|\normadj\|_{\infty}$). By depth 7, this bound becomes completely vacuous ($\sim O(10^{14})$). In contrast, our FTGC bound depends on the spectral norm of the generalized Vandermonde matrix ($||\matVp||_{2,\infty}$), which is 1 for GCNs. Consequently, our bound remains significantly tighter ($\sim O(10^5)$) and avoids exponential explosion, while still successfully capturing the generalization trends demonstrated earlier in App.~\ref{app:boundvsgap}.

\begin{figure}[htbp]
    \centering
    \begin{subfigure}{0.48\textwidth}
        \centering
        \includegraphics[width=\linewidth]{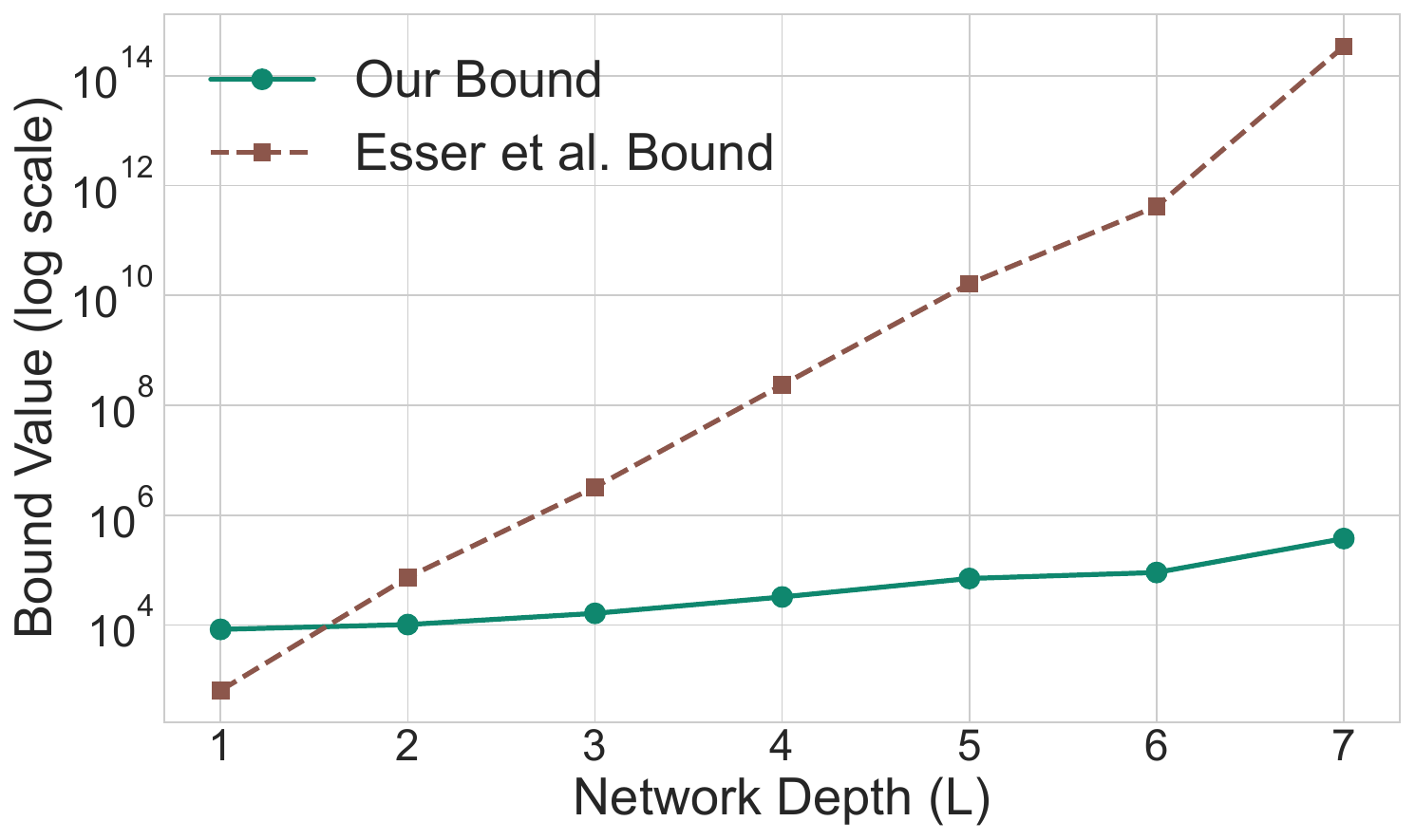}
        \caption{Cora}
    \end{subfigure}\hfill
    \begin{subfigure}{0.48\textwidth}
        \centering
        \includegraphics[width=\linewidth]{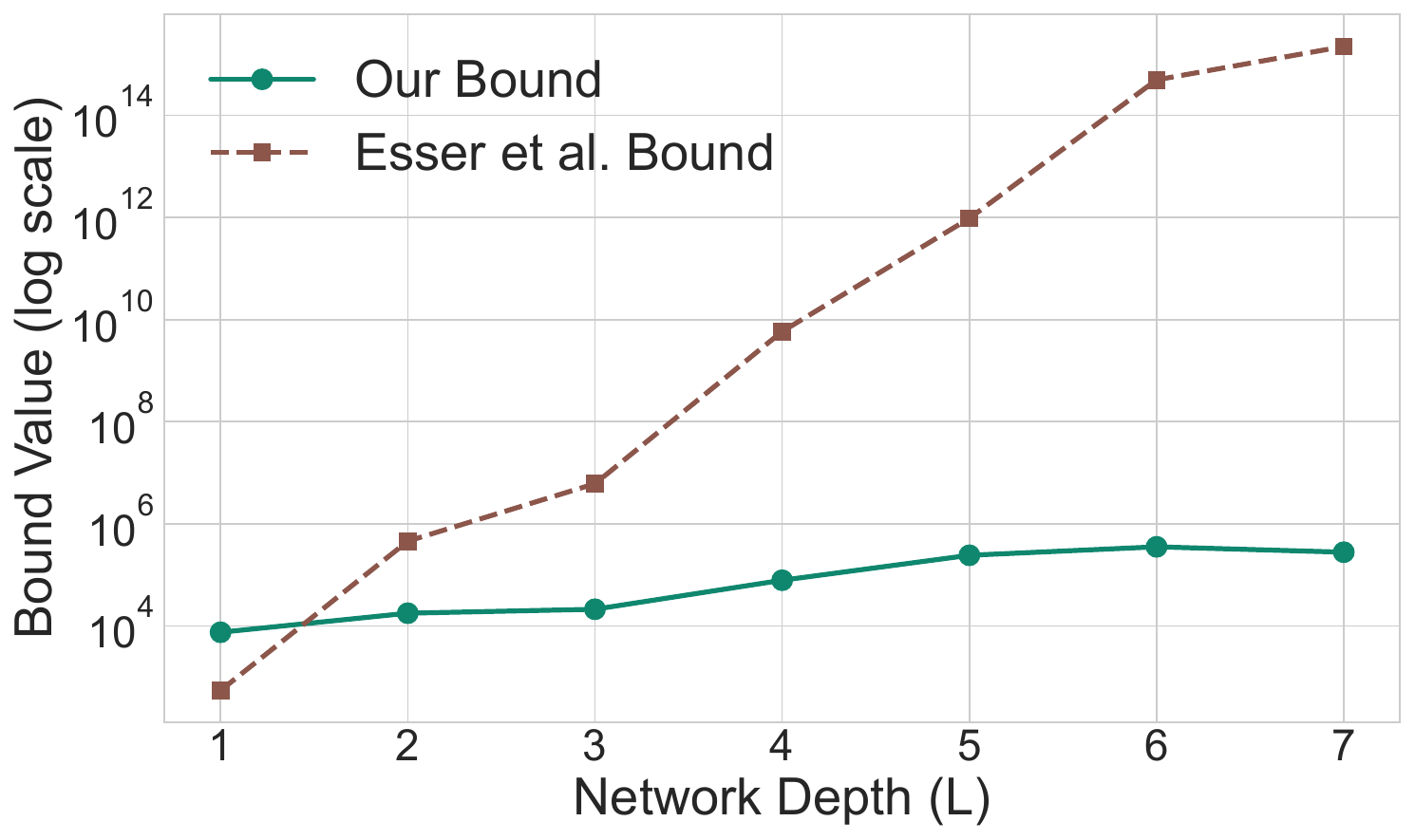}
        \caption{Chameleon}
    \end{subfigure}
    \caption{Comparison of generalization bounds.}
    \label{fig:bound_comparison}
\end{figure}

\section{Formal definitions and properties of the polynomial bases}
\label{app:fullbasisdef}

This section provides the formal definitions for the polynomial bases used in our analysis. The key property for our bounds is the filter's amplification profile, which measures the total response of the basis functions at a given frequency $x \in [-1, 1]$:
$$M_K(x) = \sum_{k=0}^K P_k(x)^2.$$
This quantity directly relates to the maximum row norm of the Vandermonde matrix $\matVp$, which appears in our generalization and stability bounds:
$$\|\matVp\|_{2,\infty} = \max_{i} \|\bm{v}(\lambda_i)\|_2 = \max_{\lambda \in \Lambda} \sqrt{M_K(\lambda)}.$$
For all four bases we consider, the amplification profile $M_K(x)$ reaches its maximum on $[-1,1]$ at the endpoints $x=\pm1$.

\paragraph{Monomial Basis.} This is the simplest polynomial basis, also known as the power basis. While straightforward, it is known to be poorly conditioned for high orders $K$, leading to numerical instability \citep{higham2002accuracy}. The basis functions are defined as:
\begin{equation}
    P_k(x) = x^k, \quad \text{for } k = 0, 1, \dots, K.
\end{equation}
Its amplification profile is a geometric series, $M_K(x) = \sum_{k=0}^K x^{2k}$, which reaches its maximum value of $M_K(\pm 1) = K+1$.

\paragraph{Chebyshev Basis.} The Chebyshev polynomials $T_k(x)$ are a sequence of orthogonal polynomials known for their  numerical stability and role in approximation theory \citep{mason2002chebyshev}. They are defined on the interval $[-1, 1]$ by the recurrence relation:
\begin{align}
    T_0(x) &= 1, \\
    T_1(x) &= x, \\
    T_{k+1}(x) &= 2x T_k(x) - T_{k-1}(x), \quad \text{for } k \ge 1.
\end{align}
Since $T_k(\pm 1)^2 = 1$ and $|T_k(x)| \leq 1$ for all $k$, the amplification profile reaches its maximum of $M_K(\pm 1) = K+1$ at the endpoints.

\paragraph{Legendre Basis.} Legendre polynomials $P_k(x)$ are another sequence of orthogonal polynomials on $[-1, 1]$, notable for being solutions to Legendre's differential equation \citep{abramowitz1964handbook}. They are defined by Bonnet's recurrence relation:
\begin{align}
    P_0(x) &= 1, \\
    P_1(x) &= x, \\
    (k+1)P_{k+1}(x) &= (2k+1)x P_k(x) - k P_{k-1}(x), \quad \text{for } k \ge 1.
\end{align}
Just like the Chebyshev polynomials, $P_k(\pm 1)^2 = 1$ and $|P_k(x)| \leq 1$, so the amplification profile has a maximum of $M_K(\pm 1) = K+1$ at the endpoints.

\paragraph{Bernstein Basis.} The Bernstein polynomials are known for their shape-preserving properties and numerical stability, particularly near interval endpoints \citep{lorentz2012bernstein}. For the spectral domain $[-1, 1]$, they are defined using an affine transformation $t = (x+1)/2$:
\begin{equation}
    P_k(x) = \binom{K}{k} t^k (1-t)^{K-k}, \quad \text{for } k = 0, 1, \dots, K.
\end{equation}
Because these polynomials are non-negative and sum to 1, their sum of squares is bounded by 1. The maximum amplification is $M_K(\pm 1) = 1$, which occurs at the endpoints where only one basis polynomial is non-zero.

\paragraph{Summary of Maximum Amplification.}
These maximum values directly determine the worst-case amplification term in our bounds. For a filter of order $K$, we get:
$$\|\matVp\|_{2,\infty} = \max_{x \in \Lambda} \sqrt{M_K(x)} \leq \max_{x \in [-1,1]} \sqrt{M_K(x)} = \begin{cases} \sqrt{K+1} & \text{for Monomial, Chebyshev, Legendre} \\ 1 & \text{for Bernstein}. \end{cases}$$
This shows that the depth-dependent terms in our bounds grow with $\sqrt{K+1}$ for the first three bases, but remain constant for the Bernstein basis.

To match worst-case amplification across bases, we can rescale
\[
\tilde P_k(x)=\frac{P_k(x)}{\sqrt{\max_{x\in [-1,1]} M_K(x)}},
\]
so that $\max_x \sum_{k=0}^K \tilde P_k(x)^2 = 1$ for all four bases and $\|\matVp\|_{2,\infty}\le 1$ (for fixed $K$). 
This keeps parameter-norm constraints comparable across bases.

\section{Hyperparameters, Datasets, and Splits}
\label{app:hyper}

\subsection{Datasets and Splits}
We conduct experiments on four public node classification datasets: the homophilic citation networks Cora and Citeseer \citep{sen2008collective}, and the heterophilic Wikipedia networks Chameleon and Squirrel \citep{pei2020geom}. For the large-scale bound validation in App.~\ref{app:boundvsgap}, we also use the OGBN-Arxiv dataset \citep{ogbdatasets}, which represents a much larger citation network. For all datasets, we use 10 random stratified splits where each split contains 10 training nodes per class, with the remaining nodes divided into a validation set (35\%) and a test set (65\%). We report the mean and 95\% confidence interval over these 10 splits.

\subsection{Model and Training}
Our model architecture implements a spectral filter layer with a residual connection, inserted between the two linear layers of an MLP. Specifically, input features are first processed by a linear layer and a ReLU activation to produce hidden feature. These features are then filtered in the spectral domain, and the result is added back to the original hidden features. A final linear layer then acts as a readout to produce the class logits. We use two distinct dropout rates, which are tuned separately. The first is applied to the input features and to the features after the spectral filtering block. The second is applied to the hidden features after the ReLU activation.

We use the Adam optimizer \citep{kingma2014adam} and train for a maximum of 2000 epochs, with early stopping triggered if the validation accuracy does not improve for 200 epochs.

\subsection{Hyperparameter Optimization}
For each dataset and polynomial basis, we conduct two separate hyperparameter searches using the Hyperopt library's Tree-structured Parzen Estimator (TPE) algorithm \citep{NIPS2011_86e8f7ab}, optimizing for the highest mean validation accuracy. The first search optimizes a baseline model where the regularizer strength $\lambda_{\text{EW}}$ is fixed at 0. The second search optimizes the fully regularized model over the complete hyperparameter space listed below.

\begin{itemize}[leftmargin=0.5cm]
    \setlength\itemsep{-2pt}
    \item \textbf{Learning Rate}: $\{0.0005, 0.001, 0.005, 0.01, 0.05, 0.1\}$
    \item \textbf{Weight Decay}: $\{0, 10^{-6}, 5 \cdot 10^{-6}, 10^{-5}, 5 \cdot 10^{-5}, 10^{-4}, 5 \cdot 10^{-4}, 10^{-3}\}$
    \item \textbf{Hidden Dimension}: $\{8, 16, 32, 64, 128\}$
    \item \textbf{Polynomial Degree}: $\{10\}$
    \item \textbf{Dropout 1 (Input/Post-Filter)}: $\{0.0, 0.2, 0.4, 0.6, 0.8\}$
    \item \textbf{Dropout 2 (Hidden)}: $\{0.0, 0.2, 0.4, 0.6, 0.8\}$
    \item \textbf{Regularization Strength} ($\lambda_{\text{EW}}$): $\{0.001, 0.005, 0.01, 0.05, 0.1, 0.5, 1, 5, 10\}$
\end{itemize}

\subsection{Code and Infrastructure}
All experiments were run on a single NVIDIA A100 GPU. The complete code to reproduce our results is available on this link: \url{https://github.com/vmart20/spectral-GNN-regularization}

\end{document}

%% file: math_commands.tex
\usepackage{amsmath,amsfonts,bm}

\def\eqref#1{equation~\ref{#1}}
\def\1{\bm{1}}

\DeclareMathAlphabet{\mathsfit}{\encodingdefault}{\sfdefault}{m}{sl}
\SetMathAlphabet{\mathsfit}{bold}{\encodingdefault}{\sfdefault}{bx}{n}